\documentclass[journal]{IEEEtran}
\usepackage{amsmath,amsfonts}
\usepackage{algorithmic}
\usepackage{algorithm}
\usepackage{array}
\usepackage[caption=false,font=normalsize,labelfont=sf,textfont=sf]{subfig}
\usepackage{textcomp}
\usepackage{stfloats}
\usepackage{url}
\usepackage{verbatim}
\usepackage{graphicx}
\usepackage{cite}
\usepackage{amssymb}
\usepackage{amsthm}
\usepackage[colorlinks=true, linkcolor=blue, citecolor=blue, urlcolor=blue]{hyperref}

\newtheorem{theorem}{Theorem}
\newtheorem{proposition}{Proposition}
\newtheorem{definition}{Definition}

\newtheorem{lemma}{Lemma}
\newtheorem{remark}{Remark}

\begin{document}

\title{Bridging the Gap between Newton--Raphson Method and Regularized Policy Iteration}

\author{Zeyang Li, Chuxiong Hu, Yunan Wang, Guojian Zhan, Jie Li, Yao Lyu, Shengbo Eben Li
    \thanks{Zeyang Li is with the Laboratory for Information and Decision Systems, Massachusetts Institute of Technology, Cambridge, MA 02139, United States.}
    \thanks{Chuxiong Hu and Yunan Wang are with the Department of Mechanical Engineering, Tsinghua University, Beijing 100084, China.}
    \thanks{Guojian Zhan, Jie Li, Yao Lyu, and Shengbo Eben Li are with the College of Artificial Intelligence and the School of Vehicle and Mobility, Tsinghua University, Beijing 100084, China.}
    \thanks{Corresponding authors: Chuxiong Hu and Shengbo Eben Li.}}

\maketitle

\begin{abstract}
    Regularization is a cornerstone of modern reinforcement learning. Regularized policy iteration (RPI) provides a fundamental scheme for solving regularized Markov decision processes (RMDPs), and the widely used soft actor-critic algorithm arises as a special case when the regularizer is Shannon entropy. Despite its empirical success, the theoretical underpinnings of RPI remain unclear. In this paper, we address this gap by proving that RPI is formally equivalent to the standard Newton--Raphson method applied to the Bellman equation smoothed by strongly convex regularizers. This equivalence enables a unified convergence analysis of existing methods and supports the development of accelerated algorithms. We show that RPI enjoys local quadratic convergence; notably, for Shannon entropy, the guarantee is dimension-free. We further study RPI with inexact policy evaluation, establishing its equivalence to an inexact Newton method in which each Newton step is solved via truncated iterations, and derive an asymptotic linear convergence rate of $\gamma^{M}$, where $M$ denotes the number of operator steps used in policy evaluation. Finally, motivated by higher-order Newton schemes, we propose a new algorithm for RMDPs that achieves third-order local convergence. Numerical experiments corroborate our theory and demonstrate the practical advantages of the proposed algorithm. Overall, our results advance the theoretical understanding of regularization in reinforcement learning and suggest new directions for efficient algorithm design.
\end{abstract}


\section{Introduction}

Reinforcement learning (RL) \cite{shengbo2018reinforcement} has achieved remarkable success in many fields that require sequential decision-making and optimal control, such as games \cite{silver2017mastering}, robotics \cite{hwangbo2019learning, huang2022reward}, and autonomous driving \cite{guan2021integrated, feng2023dense}. The foundation of RL algorithms lies in the theory of Markov decision process (MDP) \cite{puterman2014markov}. At each time step, the agent chooses an action based on its current state and gets an intermediate reward. The objective is to find an optimal policy that maximizes each state's value, i.e., the infinite horizon discounted accumulative rewards. The Bellman equation identifies the necessary and sufficient conditions of optimal values based on Bellman's principle of optimality.

Value iteration and policy iteration are two fundamental algorithms to solve the Bellman equation \cite{puterman2014markov}. Value iteration is essentially a fixed-point iteration technique by consecutively applying Bellman operator to the current value. Policy iteration alternately performs two steps: policy evaluation and policy improvement. Puterman and Brumelle \cite{puterman1979convergence} show that policy iteration is a particular variant of Newton--Raphson method applied to Bellman equation. Since Bellman equation contains the max operator, it is nonsmooth and cannot be directly solved by Newton--Raphson method in which its derivative is required. In policy iteration, this difficulty is circumvented by locally linearizing the Bellman equation at the current value, with the derivative computed from the resulting locally linearized equation \cite{puterman1979convergence}. Puterman and Brumelle prove that policy iteration enjoys global linear convergence with the rate being $\gamma$ (discount factor). More recently, it is shown that policy iteration achieves asymptotic quadratic convergence \cite{bertsekas2022lessons, gargiani2022dynamic} under certain nondegeneracy conditions. A widely used variant of policy iteration is the modified policy iteration algorithm, in which repeated cycles of policy improvement and inexact policy evaluation are performed \cite{puterman2014markov}, \cite{puterman1978modified}, \cite{shengbo2018reinforcement}. Modified policy iteration interpolates between value iteration and policy iteration: it reduces to value iteration when the evaluation uses a single operator step, and approaches policy iteration in the limit of infinitely many steps. Puterman and Shin prove the convergence of modified policy iteration under the assumption that the initial value is element-wise smaller than the optimal value. Scherrer et al. \cite{scherrer2012approximate} show that modified policy iteration has global linear convergence with the rate being $\gamma$ and provide comprehensive error propagation analysis for its approximate form.

Regularization techniques are frequently used in contemporary RL algorithms for various purposes, such as encouraging exploration and improving robustness. These approaches incorporate a regularizer into the optimization objective, thereby modifying the standard MDP formulation. Shannon entropy is arguably the most prominent regularizer. Haarnoja et al. \cite{haarnoja2017reinforcement} propose soft Q-learning by integrating Shannon entropy into value iteration, and subsequently develop a soft policy iteration framework that leads to the soft actor-critic algorithm \cite{haarnoja2018soft}. Duan et al. \cite{duan2021distributional} combine distributional value learning with soft policy iteration to obtain the distributional soft actor-critic algorithm. Srivastava et al. \cite{srivastava2021parameterized} use Shannon entropy to quantify exploration and derive stochastic policies that maximize this measure while ensuring a small expected cost. Tsallis entropy is another notable regularizer. Lee et al. \cite{lee2018sparse} construct a sparse MDP with Tsallis entropy regularization and propose a sparse value iteration algorithm. Chow et al. \cite{chow2018path} employ Tsallis entropy in path consistency learning and derive a sparse consistency equation, while Lee et al. \cite{lee2020generalized} develop Tsallis actor-critic by incorporating Tsallis entropy into policy iteration. While these algorithms share a common reliance on regularization, they arise from disparate motivations and often rely on ad hoc analyses. Geist et al. \cite{geist2019theory} unify a broad class of regularizers within the framework of regularized MDPs. The key idea is to define a regularized Bellman operator by augmenting the standard Bellman operator with a strongly convex function (the regularizer) via the Legendre--Fenchel transform \cite{nesterov2005smooth, mensch2018differentiable}. Since regularized Bellman operator shares the same properties with its unregularized counterparts, such as contraction and monotonicity \cite{geist2019theory}, the dynamic programming techniques in standard MDPs (i.e, policy iteration and modified policy iteration) can be utilized to solve regularized MDPs.

Despite the development of the regularized MDP framework, convergence guarantees for the corresponding algorithms remain relatively limited. Geist et al. \cite{geist2019theory} establish global linear convergence of regularized policy iteration with rate $\gamma$ by exploiting the monotone contraction property of the regularized Bellman operators. More recently, Cen et al. \cite{cen2022fast} show that soft policy iteration (a special case of regularized policy iteration when choosing Shannon entropy as the regularizer) can achieve asymptotic quadratic convergence. Their result, however, relies on a nondegeneracy assumption on the optimal state distribution and yields convergence bounds that depend on the dimension of the state space.

Motivated by analyses in the unregularized setting  \cite{puterman1979convergence, bertsekas1997nonlinear, gargiani2022dynamic}, this paper establishes a formal equivalence between regularized policy iteration and the standard Newton--Raphson method. This equivalence not only facilitates a unified analysis of existing methods for regularized MDPs but also inspires the design of novel algorithms with accelerated convergence. A visual overview of the main results is provided in Fig.~\ref{main results}. The main contributions of this paper are summarized as follows.
\begin{itemize}
    \item This paper proves that regularized policy iteration is strictly equivalent to the standard Newton--Raphson method applied to the Bellman equation smoothed by strongly convex functions. To the best of our knowledge, this is the first result to reveal this formal connection. The key idea is that smoothed Bellman equation can be converted into an equivalent affine transformation form, in which the Jacobian serves as the linear map. This enables the Newton iteration formula to be simplified into a regularized self-consistency equation, which corresponds to the policy evaluation part of regularized policy iteration.
    
    \item This paper proves that regularized policy iteration converges quadratically in a local region around the optimal value. The key to the proof is to exploit the global Lipschitz continuity of the Jacobian of the smoothed Bellman equation and to bound the difference terms in the Newton iteration formula. This result sheds light on the role of regularization in enabling fast convergence. To the best of our knowledge, this is the first quadratic convergence result for regularized policy iteration with general strongly convex functions. When the regularizer is Shannon entropy, the convergence guarantee is dimension-free and requires no nondegeneracy assumptions, representing a substantial improvement over existing results for soft policy iteration \cite{cen2022fast}.

    \item We further extend the analysis from regularized policy iteration to the case with inexact policy evaluation. This variant, known as regularized modified policy iteration, has been discussed by Scherrer et al. \cite{scherrer2012approximate} and Geist et al. \cite{geist2019theory}. We show that regularized modified policy iteration is equivalent to an inexact Newton method in which each Newton step is solved via truncated iterations. We prove that the asymptotic convergence rate of regularized modified policy iteration is $\gamma^M$, where $M$ denotes the number of operator steps used in policy evaluation. The result follows by showing that the evaluation error decays at rate $\gamma^M$ with respect to a norm naturally tied to the optimal value.
    
    \item We develop a new algorithm for regularized MDPs inspired by higher-order Newton schemes \cite{shamanskii1967modification, petkovic2014multipoint}, which achieves third-order local convergence. We demonstrate that this third-order method outperforms regularized policy iteration in wall-clock time in numerical experiments, highlighting the practical benefits of the proposed algorithm and further justifying the algorithmic value of the Newton--Raphson perspective on regularized policy iteration.
\end{itemize}

\begin{figure*}[ht]
    \centering
    \includegraphics[width=5.6in]{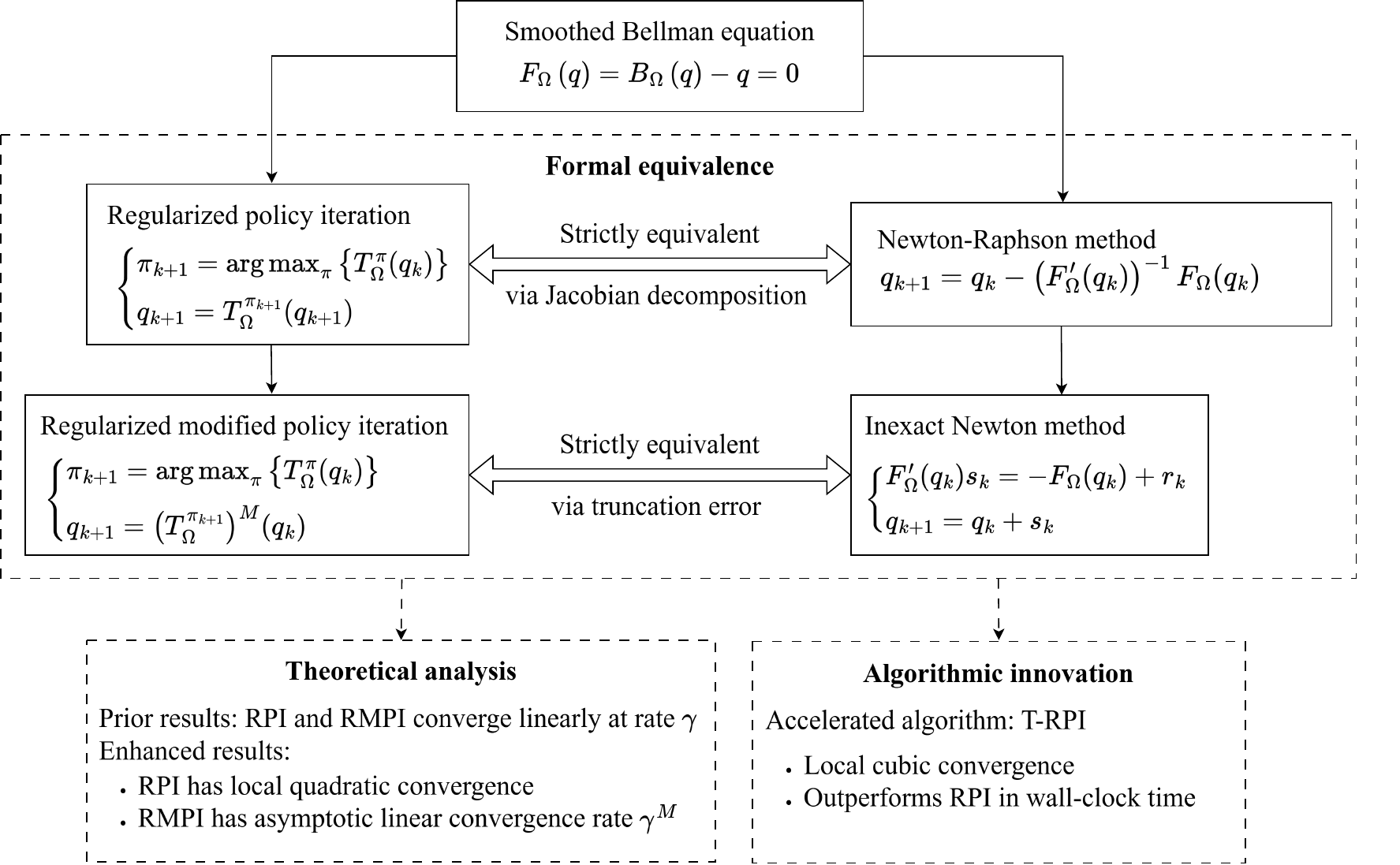}
    \caption{A visual overview of the main results in this paper. We prove that regularized policy iteration (RPI) is formally equivalent to Newton--Raphson method, and that regularized modified policy iteration (RMPI) corresponds to the inexact Newton method. These equivalences provide a unified framework for analyzing the convergence of existing algorithms, yield stronger results, and motivate new methods with accelerated convergence.}
    \label{main results}
\end{figure*}

\section{Preliminaries}

\subsection{Regularized Markov Decision Processes}
A standard MDP can be represented by a 5-tuple $\{\mathcal{S},\mathcal{A},P,r,\gamma \}$, where $\mathcal{S}$ represents the finite state space, $\mathcal{A}$ represents the finite action space, $\gamma$ denotes the discount factor, $P$ represents the Markov transition kernel and $r$ represents the reward function. A policy $\pi: \mathcal{S} \rightarrow \Delta_{|\mathcal{A}|}$ maps a state to a probability distribution over action space.

For a given policy $\pi$, its state-value function $v_\pi: \mathcal{S} \rightarrow \mathbb{R}$ is defined as the expectation of discounted cumulative reward starting with state $s$. The action-value function $q_\pi: \mathcal{S} \times \mathcal{A} \rightarrow \mathbb{R}$ of a policy $\pi$ is defined as the expectation of discounted cumulative reward starting with state $s$ and action $a$. Value functions naturally hold a recursive structure due to the Markovian property and the infinite horizon setting, which is referred to as the self-consistency conditions. Based on Bellman’s principle of optimality, the Bellman equation denotes the necessary conditions for optimal values. See \cite{shengbo2018reinforcement} for more details. In this paper, we mainly discuss the equations and algorithms for action-value function $q(s,a)$ for simplicity of expression. Similar results can be easily derived for state-value function $v(s)$.

In the tabular setting (finite state and action spaces), the action-value for all state-action pairs can be viewed as vectors in Euclidean spaces, i.e., $q \in \mathbb{R}^{|\mathcal{S}| \cdot |\mathcal{A}|}$. The self-consistency condition and Bellman equation can be viewed as systems of multivariable equations. Two operators are defined to further simplify the notation. The self-consistency operator $T^{\pi}$ is defined as the right-hand side of self-consistency condition:
\begin{equation}
    \label{self-consistency operator}
    \begin{aligned}
        &[T^{\pi}(q)](s, a)=r(s, a)\\&+\gamma \sum_{s^{\prime} \in \mathcal{S}} p\left(s^{\prime} | s, a\right) \sum_{a^{\prime} \in \mathcal{A}} \pi\left(a^{\prime} | s^{\prime}\right) q\left(s^{\prime}, a^{\prime}\right),
    \end{aligned}
\end{equation}
in which $p\left(s^{\prime} | s, a\right)$ denotes the state transition probability following the Markov transition kernel $P$. Note that $T^{\pi}$ is defined with a particular policy $\pi$. The Bellman operator $B$ is defined as the right-hand side of Bellman equation:
\begin{equation}
    \label{original Bellman equation}
    [B\left( q \right)](s,a) =\max_{\pi} \left\{ [T^{\pi}(q)](s,a) \right\}.
\end{equation}

In the modified policy iteration (MPI) algorithm, policy improvement and inexact policy evaluation are alternately performed \cite{puterman2014markov}, \cite{scherrer2012approximate}:
\begin{equation}
    \label{M-PI}
    \left\{\begin{array}{l}
        \pi_{k+1}=\underset{\pi}{\mathrm{argmax}}\left\{ T^{\pi}\left( q_k\right) \right\} \\
        q_{k+1}=\left(T^{\pi_{k+1}}\right)^M (q_k)
    \end{array}\right. \quad.
\end{equation}
MPI reduces to value iteration as $M=1$ and becomes policy iteration as $M=\infty$.

Geist et al. propose the framework of regularized Markov decision processes (RMDPs), in which the core idea is to regularize the original operators \cite{geist2019theory}. Let $\Omega: \Delta_{|\mathcal{A}|} \rightarrow \mathbb{R}$ be a strongly convex function. The regularized self-consistency operator is defined as:
\begin{equation}
    \label{regularized self-consistency operator}
    \begin{aligned}
        &[T^{\pi}_{\Omega}(q)](s, a)=r(s, a)\\
        &+\gamma \sum_{s^{\prime} \in \mathcal{S}} p\left(s^{\prime} | s, a\right) \left(\sum_{a^{\prime} \in \mathcal{A}} \pi\left(a^{\prime} | s^{\prime}\right) q\left(s^{\prime}, a^{\prime}\right) - \lambda \Omega(\pi(\cdot|s^{\prime}))\right).
    \end{aligned}
\end{equation}
in which $\pi(\cdot | s) = [\pi(a_1 | s), \pi(a_2 | s),\cdots,\pi(a_{|\mathcal{A}|} | s)]^\top$, and $\lambda$ is the regularization coefficient. The regularized Bellman operator is defined as:
\begin{equation}
    \label{regularized Bellman operator}
    [B_{\Omega}\left( q \right)](s,a) =\max_{\pi}\left\{ [T^{\pi}_{\Omega}(q)](s,a) \right\}.
\end{equation}

Following the idea of MPI in standard MDPs, regularized modified policy iteration (RMPI) is proposed to solve RMDPs \cite{geist2019theory}:
\begin{equation}
    \label{MR-PI}
    \left\{\begin{array}{l}
        \pi_{k+1}=\underset{\pi}{\mathrm{argmax}}\left\{ T^{\pi}_{\Omega}\left( q_k \right) \right\} \\
        q_{k+1}=\left(T^{\pi_{k+1}}_{\Omega}\right)^M (q_k)
    \end{array}\right. \quad.
\end{equation}
RMPI reduces to regularized value iteration (RVI) as $M=1$ and becomes regularized policy iteration (RPI) as $M=\infty$.

\subsection{Newton--Raphson Method}
In numerical analysis, the Newton--Raphson (NR) method is a widely-used algorithm for solving nonlinear equations. Consider a system of nonlinear equations, $F(x) = 0$, where $F: \mathbb{R}^n \rightarrow \mathbb{R}^n$ is a vector-valued function. Given $x_0 \in \mathbb{R}^n$, NR method generates a sequence $\left\{ x_k \right\}$ converging to a solution of $F(x) = 0$:
\begin{equation}
    \label{Newton Iteration}
    x_{k+1} = x_{k}- {F^{\prime}(x_k)}^{-1} F(x_k),
\end{equation}
in which $F^{\prime}(x_k)$ denotes the Jacobian of $F(x)$ at $x_k$ \cite{ortega2000iterative}.

At each iteration step, a system of linear equations (\ref{Newton Iteration}) is solved. In practical implementations of NR method, it is solved approximately, which naturally gives rise to the class of inexact Newton methods:
\begin{equation}
    \nonumber
    \left\{ \begin{array}{l}
        F^{\prime}(x_k)s_k=-F(x_k)+r_k\\
        x_{k+1}=x_k+s_k\\
    \end{array} \right. \quad,
\end{equation}
in which $r_k$ denotes the error of solving the Newton iteration formula (\ref{Newton Iteration}) \cite{dembo1982inexact}.

\section{Smoothed Bellman Equation}

Bellman equation identifies the necessary and sufficient conditions of optimal values based on Bellman's principle of optimality. The nonsmooth nature of the max operator makes the Bellman equation tricky to solve. Nesterov originally proposes the idea of smoothing the nonsmooth function in optimization \cite{nesterov2005smooth}, making use of Legendre--Fenchel transform (i.e., the convex conjugate) \cite{boyd2004convex}. Following Nesterov's method, smoothed max operator can be obtained by adding a strongly convex function into the formulation of max operator. On the other hand, in the RL community, researchers have proposed various algorithms by adding different regularizers into Bellman equation (see Introduction for detailed enumeration). Geist et al. \cite{geist2019theory} identify the connections between regularization and Nesterov's smoothed max operator and propose the framework of regularized MDPs, showing regularization in RL is essentially a smoothing technique applied on the max operator inside Bellman equation.

Despite the discovery that regularization is a smoothing technique and induces the smoothed Bellman equation, the connections between the iterative algorithms and the equation-solving techniques remain unclear. This paper shows for the first time that RPI is strictly equivalent to the standard NR method applied to smoothed Bellman equation. In this section, we prove some important properties of smoothed Bellman equation, which is crucial for establishing the connections between RPI and NR method. We begin with defining smoothed max operator and stating its properties.

\begin{definition}[smoothed max operator \cite{nesterov2005smooth}, \cite{mensch2018differentiable}]
    \label{smoothed max operator}
    Given $x\in \mathbb{R} ^m $ and a strongly convex function $\Omega: \mathbb{R} ^m \rightarrow \mathbb{R}$, the smoothed max operator induced by $\Omega$ is defined as:
    \begin{equation}
        {\max} _{\Omega}(x)\triangleq \max_{p\in \Delta_m} \left\{ \langle p,x\rangle -\lambda\Omega (p) \right\} ,
        \label{definition of smoothed max operator}
    \end{equation}
    in which $\Delta_m$ denotes the probability simplex, i.e., $\Delta_m=\left\{ p\in \mathbb{R} ^m \mid \sum_{i=1}^m p_i=1, p\geq 0 \right\}$ and the regularization parameter $\lambda>0$ controls the degree of smoothing. $\langle \cdot ,\cdot \rangle$ denotes the inner product between two vectors.
\end{definition}

\begin{lemma}[properties of smoothed max operator]
    \hfill
    \begin{enumerate}
        
        \item Suppose $\Omega$ is $\mu$-strongly convex with respect to a norm $\|\cdot\|$, whose dual norm is denoted as $\|\cdot\|_*$. The gradient of smoothed max operator $\nabla {\max} _{\Omega}(x)$ satisfies
            \begin{equation}
                \label{Lipschitz continuous gradient of smoothed max operator}
                \Vert\nabla {\max} _{\Omega}(x)-\nabla {\max} _{\Omega}(y)\Vert \leq \frac{1}{\lambda\mu}\Vert x-y \Vert_*.
            \end{equation}
        
        \item The gradient of smoothed max operator is given by
            \begin{equation}
                \label{gradient of smoothed max operator}
                \nabla {\max} _{\Omega}(x)=\underset{p\in \Delta_m}{\mathrm{argmax}}\left\{ \langle p,x\rangle -\lambda\Omega (p) \right\}.
            \end{equation}
            Besides, ${[\nabla {\max} _{\Omega}(x)]}_i \geq 0$ and $ \sum_{i=1}^m {[\nabla {\max} _{\Omega}(x)]}_i =1 $ hold.

        \item The smoothed max operator can be decomposed with its gradient, i.e.,
            \begin{equation}
                \label{decomposition of smoothed max operator}
                {\max} _{\Omega}(x) = \langle \nabla {\max} _{\Omega}(x), x \rangle - \lambda\Omega(\nabla {\max} _{\Omega}(x)).
            \end{equation}
    \end{enumerate}
\end{lemma}

\begin{proof}
    \hfill
    \begin{enumerate}
        
        \item Note that by definition, smoothed max operator ${\max} _{\Omega}(x)$ is the conjugate of $\lambda\Omega(p)$. \eqref{Lipschitz continuous gradient of smoothed max operator} is a standard result for conjugate functions in convex analysis \cite{borwein2006convex}.
        
        \item Define $p^{*}(x)=\underset{p\in \Delta_m}{\mathrm{argmax}}\left\{ \langle p,x\rangle -\lambda\Omega (p) \right\}$ for any $x\in \mathbb{R}^m$. For any $x,y\in \mathbb{R}^m$, we have ${\max} _{\Omega}(y) \geq \langle p^{*}(x),y\rangle-\lambda\Omega (p^{*}(x))=\langle p^{*}(x),x\rangle-\lambda\Omega (p^{*}(x))+\langle p^{*}(x),y-x\rangle={\max} _{\Omega}(x)+\langle p^{*}(x),y-x\rangle$, so $p^{*}(x)$ is a subgradient of ${\max} _{\Omega}(x)$. Since ${\max} _{\Omega}(x)$ is the conjugate of strongly convex function $\lambda\Omega(p)$, it is differentiable. For a differentiable function, its subgradient is unique and equals the gradient.
        
        \item Since the maximum in (\ref{definition of smoothed max operator}) is the gradient $\nabla {\max} _{\Omega}(x)$, (\ref{decomposition of smoothed max operator}) is a direct result by substituting the maximum back in (\ref{definition of smoothed max operator}).
    \end{enumerate}
\end{proof}

\begin{remark}[examples for smoothed max operators]
    \label{examples of regularization-based smoothed max operators}
    If the regularizer is chosen as (minus) Shannon entropy, i.e., $\Omega (p)=\sum_i{p_i\log p_i}$, the corresponding smoothed max operator has a closed-form expression (log-sum-exp, see \cite{boyd2004convex}): ${\max} _{\Omega}(x)=\lambda\log \left( \sum_i{e^{\frac{1}{\lambda} x_i}}\right)$. Its gradient is the softmax function: ${[\nabla {\max} _{\Omega}(x)]}_i=e^{\frac{1}{\lambda} x_i}/\left({\sum_i{e^{\frac{1}{\lambda} x_i}}}\right)$. Shannon entropy is utilized in soft policy iteration and soft actor-critic \cite{haarnoja2018soft}, which is a special case of regularized policy iteration. Since Shannon entropy is $1$-strongly convex with respect to $\Vert \cdot \Vert_1$ over the probability simplex, we have $\Vert\nabla {\max} _{\Omega}(x)-\nabla {\max} _{\Omega}(y)\Vert_{1} \leq \frac{1}{\lambda}\Vert x-y \Vert_{\infty}$.

    If the regularizer is chosen as (minus) second-order Tsallis entropy, i.e., $\Omega (p)=\frac{1}{2}(\sum_i{p_i^2}-1)$, the corresponding smoothed max operator also has a (rather complicated) closed-form expression, whose gradient is the sparsemax function (see \cite{lee2018sparse, chow2018path, lee2020generalized} for more details). This choice is $1$-strongly convex with respect to $\Vert \cdot \Vert_2$ and $\Vert\nabla {\max} _{\Omega}(x)-\nabla {\max} _{\Omega}(y)\Vert_{2} \leq \frac{1}{\lambda}\Vert x-y \Vert_{2}$ holds.
\end{remark}

By replacing the max operators inside Bellman equation with smoothed max operators, we obtain a smooth approximation of Bellman equation. We formally define the smoothed Bellman equation in matrix form as follows.

\begin{definition}[smoothed Bellman equation]
    Suppose we have $|\mathcal{S}|=n$ and $|\mathcal{A}|=m$. The action-value $q$ denotes a vector in $\mathbb{R}^{nm}$ containing all the values for each state-action pair $(s,a)$. Smoothed Bellman equation is a multivariate nonlinear equation on $\mathbb{R}^{nm}$:
    \begin{equation}
        \label{smoothed Bellman equation}
        F_{\Omega}\left( q \right) =B_{\Omega}\left( q \right) -q = \gamma P f_{\Omega}(q) +r - q =0.
    \end{equation}
    The vector $r\in\mathbb{R}^{nm}$ denotes all the rewards for each state-action pair $(s,a)$. The matrix $P\in \mathbb{R}^{nm\times n}$ is determined by the Markov transition kernel:
    \begin{equation}
        P=\left(\begin{array}{c}
        P\left(s_1\right) \\
        P\left(s_2\right) \\
        \vdots \\
        P\left(s_n\right)
        \end{array}\right),
        \nonumber
    \end{equation}
    in which $P\left( s_k \right) \in \mathbb{R}^{m\times n}$ denotes all the transition probabilities starting from $s_k$:
    \begin{equation}
        P\left( s_k \right) =\left( \begin{matrix}
            p\left( s_1|s_k,a_1 \right)&		\cdots&		p\left( s_n|s_k,a_1 \right)\\
            p\left( s_1|s_k,a_2 \right)&		\cdots&		p\left( s_n|s_k,a_2 \right)\\
            \vdots&		\vdots&		\vdots\\
            p\left( s_1|s_k,a_m \right)&		\cdots&		p\left( s_n|s_k,a_m \right)\\
        \end{matrix} \right).
        \nonumber
    \end{equation}
    $f_{\Omega}(q)$ is given by
    \begin{equation}
        f_{\Omega}(q) =\left( \begin{array}{c}
                \max_{\Omega} \left\{ q(s_1,\cdot) \right\}\\
                \max_{\Omega} \left\{ q(s_2,\cdot) \right\}\\
                \vdots\\
                \max_{\Omega} \left\{ q(s_n,\cdot) \right\}\\
        \end{array} \right).
        \nonumber
    \end{equation}
\end{definition}

The Jacobian of smoothed Bellman equation (\ref{smoothed Bellman equation}) can be calculated directly as
\begin{equation}
    F^{\prime}_{\Omega}(q)=\gamma P \nabla f_{\Omega}(q) - I,
    \label{Jacobian of smoothed Bellman equation}
\end{equation}
where $I$ denotes the identity matrix of size $n m$ and $\nabla f_{\Omega}(q) \in \mathbb{R}^{n \times nm}$ is a block diagonal matrix given by
\begin{equation}
    \label{gradient of f_Omega}
    \begin{aligned}
        &\nabla f_{\Omega}(q) \\
        =&\left( \begin{matrix}
            (\nabla \max_{\Omega} \left\{ q(s_1,\cdot ) \right\} )^\top&		&		\\
            &		\ddots&		\\
            &		&		(\nabla \max_{\Omega} \left\{ q(s_n,\cdot ) \right\} )^\top\\
        \end{matrix} \right).
    \end{aligned}
\end{equation}

We state the properties of Jacobian (\ref{Jacobian of smoothed Bellman equation}) in the following propositions. Proposition \ref{properties of the Jacobian's inverse} shows that the inverse of Jacobian always exists and is negative and bounded. Proposition \ref{Lipschitz continuity of the Jacobian} shows that the Jacobian is Lipschitz continuous on $\mathbb{R}^{nm}$. Proposition \ref{proposition on decomposition with Jacobian} shows that smoothed Bellman equation can be decomposed with its Jacobian, i.e., it has an equivalent affine transformation form in which the Jacobian serves as the linear map.

\begin{proposition}[properties of the Jacobian's inverse]
    \label{properties of the Jacobian's inverse}
    The inverse of Jacobian of smoothed Bellman equation (\ref{Jacobian of smoothed Bellman equation}) always exists and is negative. Furthermore, the infinity norm of the inverse of Jacobian is bounded by
    \begin{equation}
        {\Vert F_{\Omega}^{\prime}(q)^{-1} \Vert}_{\infty} \leq \frac{1}{1-\gamma}.
        \nonumber
    \end{equation}
\end{proposition}

\begin{proof}
    Note that ${\Vert P \Vert}_{\infty}=1$ holds. Using the properties of smoothed max operator, we also have ${\Vert \nabla f_{\Omega}(q) \Vert}_{\infty}=1$. Each row of $ \nabla f_{\Omega}(q)$ sums to $1$, so ${\Vert P \nabla f_{\Omega}(q) \Vert}_{\infty}=1$ holds. Since ${\Vert \gamma P \nabla f_{\Omega}(q) \Vert}_{\infty}=\gamma < 1$, using the Neumann series of matrices, the inverse of the Jacobian $F_{\Omega}^{\prime}(q)$ is given by
    \begin{equation}
        \label{explicit formula for inverse of the Jacobian}
        \begin{aligned}
        F_{\Omega}^{\prime}(q)^{-1} &= (\gamma P \nabla f_{\Omega}(q) - I)^{-1}\\
        & = - \sum\limits_{i=0}^{\infty} {(\gamma P \nabla f_{\Omega}(q))^i}.
        \end{aligned}
    \end{equation}
    Since $\gamma P \nabla f_{\Omega}(q) > 0$, we have $F_{\Omega}^{\prime}(q)^{-1} < -I $. Taking the infinity norm of $F_{\Omega}^{\prime}(q)^{-1}$, we obtain
    \begin{equation}
        \begin{aligned}
            \begin{aligned}
                \left\|F_{\Omega}^{\prime}(q)^{-1}\right\|_{\infty} \leq \sum_{i=0}^{\infty}\left(\left\|\left(\gamma P \nabla f_{\Omega}(q)\right)\right\|_{\infty}\right)^i  \leq \sum_{i=0}^{\infty} \gamma^i=\frac{1}{1-\gamma} .
            \end{aligned}
        \end{aligned}
        \nonumber
    \end{equation}
\end{proof}

\begin{proposition}[Lipschitz continuity of the Jacobian]
    \label{Lipschitz continuity of the Jacobian} Assume the regularizer $\Omega$ is $\mu$-strongly convex with respect to $\Vert \cdot \Vert_1$.
    The Jacobian of smoothed Bellman equation (\ref{Jacobian of smoothed Bellman equation}) is Lipschitz continuous with respect to $\Vert \cdot \Vert_{\infty}$, i.e.,
    \begin{equation}
        \label{expression of Lipschitz continuity of the Jacobian}
        \Vert F^{\prime}_{\Omega}(q_1)-F^{\prime}_{\Omega}(q_2)\Vert _{\infty}\leq \frac{\gamma}{\lambda\mu}\Vert q_1-q_2\Vert _{\infty},
    \end{equation}
    holds for all $q_1,q_2 \in \mathbb{R}^{nm}$.
\end{proposition}

\begin{proof}
    The infinity norm of $\nabla f_{\Omega}(q_1)-\nabla f_{\Omega}(q_2)$ is bounded by
    \begin{equation}
        \begin{aligned}
            \Vert &\nabla f_{\Omega}(q_1)-\nabla f_{\Omega}(q_2)\Vert_{\infty}\\& = \max_i \left\{ \left\| \nabla {\max}_{\Omega} \left\{ q_1(s_i,\cdot ) \right\} - \nabla {\max}_{\Omega} \left\{ q_2(s_i,\cdot ) \right\}  \right\| _{1} \right\} \\ & \leq \max _i\left\{\frac{1}{\lambda\mu}\left\|q_1\left(s_i, \cdot\right)-q_2\left(s_i, \cdot\right)\right\|_{\infty}\right\} \\ & \leq \frac{1}{\lambda\mu}\left\|q_1-q_2\right\|_{\infty},
        \end{aligned}
        \nonumber
    \end{equation}
    for all $q_1,q_2 \in \mathbb{R}^{nm}$, where the first inequality follows from \eqref{Lipschitz continuous gradient of smoothed max operator}.
    Since $\Vert P \Vert_{\infty}=1$, we have
    \begin{equation}
        \begin{aligned}
            \Vert F^{\prime}_{\Omega}(q_1)-F^{\prime}_{\Omega}(q_2)\Vert _{\infty}&\leq \gamma \Vert P \Vert_{\infty} \Vert \nabla f_{\Omega}(q_1)-\nabla f_{\Omega}(q_2)\Vert_{\infty}\\&\leq \frac{\gamma}{\lambda\mu}\left\|q_1-q_2\right\| _{\infty},
        \end{aligned}
        \nonumber
    \end{equation}
    which completes the proof.
\end{proof}

\begin{proposition}[decomposition with Jacobian]
    \label{proposition on decomposition with Jacobian}
    The following equation is an equivalent form of smoothed Bellman equation (\ref{smoothed Bellman equation}):
    \begin{equation}
        \label{equivalent smoothed Bellman equation}
        \begin{aligned}
            F_{\Omega}(q) &=F_{\Omega}^{\prime}(q)q+\gamma P e_{\Omega}(q)+r\\&= \gamma P (\nabla f_{\Omega}\left( q \right) q+e_{\Omega}\left( q \right)) +r-q,
        \end{aligned}
    \end{equation}
    in which $[e_{\Omega}\left( q \right)]_i = -\lambda \Omega\left(\nabla \max _{\Omega}\left(q\left(s_i, \cdot\right)\right)\right)$ and $e_{\Omega}\left( q \right) \in \mathbb{R}^n$.
\end{proposition}

\begin{proof}
    Utilizing the decomposition of smoothed max operator (\ref{decomposition of smoothed max operator}), we have
    \begin{equation}
        \begin{aligned}
            & [f_{\Omega}(q)]_i = {\max} _{\Omega}\left\{q\left(s_i, \cdot\right)\right\} \\ &= \langle\nabla {\max} _{\Omega}\left(q\left(s_i, \cdot\right)\right), q\left(s_i, \cdot\right)\rangle-\lambda \Omega\left(\nabla {\max} _{\Omega}\left(q\left(s_i, \cdot\right)\right)\right).
        \end{aligned}
    \nonumber
    \end{equation}
    It can be verified that $f_{\Omega}\left( q \right)$ has the matrix expression:
    \begin{equation}
        f_{\Omega}(q) = \nabla f_{\Omega}\left( q \right) q+e_{\Omega}\left( q \right).
        \nonumber
    \end{equation}
    Substitute it back into smoothed Bellman equation (\ref{smoothed Bellman equation}) and (\ref{equivalent smoothed Bellman equation}) is obtained.
\end{proof}

\section{Connections between RPI and NR Method}

In this section, we identify the connections between RPI algorithms in reinforcement learning and NR methods in numerical analysis. We show that RPI is equivalent to the standard NR method applied to smoothed Bellman equation. We also show that RMPI corresponds to inexact Newton method in which the Newton iteration formula is solved with truncated iterations.

\subsection{Equivalence between RPI and NR method}

We prove the equivalence between RPI and NR method in the following theorem. We show that $q_{k+1}$ obtained by one-step NR method on smoothed Bellman equation equals $q_{\pi_{k+1}}$ obtained by one-step RPI, starting from the same point $q_k=q_{\pi_k}$. The key method is to utilize the decomposition of smoothed Bellman equation with its Jacobian, which enables the Newton iteration formula to be simplified into a regularized self-consistency condition.

\begin{theorem}[equivalence between RPI and NR method]
    \label{R-PI as standard NR method}
    Suppose $q_k=q_{\pi_k}$. $q_{\pi_{k+1}}$ is obtained by carrying one-step RPI from $q_{\pi_k}$ as is described in (\ref{MR-PI}) with $M=\infty$:
    \begin{equation}
        \label{R-PI}
        \left\{\begin{array}{l}
            \pi_{k+1}=\underset{\pi}{\mathrm{argmax}}\left\{ T^{\pi}_{\Omega}\left( q_{\pi _k} \right) \right\} \\
            q_{\pi_{k+1}}=\left(T^{\pi_{k+1}}_{\Omega}\right)^{\infty} (q_{\pi_{k}})
        \end{array}\right. \quad.
    \end{equation}
    $q_{k+1}$ is obtained by carrying out one-step NR iteration from $q_k$ as is described in (\ref{Newton Iteration}):
    \begin{equation}
        \label{Newton Iteration on smoothed Bellman equation}
        q_{k+1} = q_{k}- {F_{\Omega}^{\prime}(q_k)}^{-1} F_{\Omega}(q_k).
    \end{equation}
    Then we have $q_{k+1}=q_{\pi_{k+1}}$, i.e., RPI is equivalent to the standard NR method applied to smoothed Bellman equation (\ref{smoothed Bellman equation}).
\end{theorem}

\begin{proof}
    Utilizing the Newton iteration (\ref{Newton Iteration on smoothed Bellman equation}), we have
    \begin{equation}
        F_{\Omega}^{\prime}\left(q_k\right)\left(q_k-q_{k+1}\right)=F_{\Omega}\left(q_k\right).
        \nonumber
    \end{equation}
    Utilizing the decomposition of smoothed Bellman equation (\ref{equivalent smoothed Bellman equation}), we have
    \begin{equation}
        \label{decomposition of smoothed Bellman equation at q_k}
        \left\{ \begin{array}{l}
            F_{\Omega}\left( q_k \right) =\gamma P\nabla f_{\Omega}\left( q_k \right) q_k+\gamma Pe_{\Omega}\left( q_k \right) +r-q_k\\
            F_{\Omega}^{\prime}\left( q_k \right) =\gamma P\nabla f_{\Omega}\left( q_k \right) -I\\
        \end{array} \right. .
    \end{equation}
    Substituting it into the previous formula, we obtain
    \begin{equation}
        \begin{aligned}
            &\left( \gamma P\nabla f_{\Omega}\left( q_k \right) -I \right) q_k-\left( \gamma P\nabla f_{\Omega}\left( q_k \right) -I \right) q_{k+1}\\&=\gamma P\nabla f_{\Omega}\left( q_k \right) q_k+\gamma Pe_{\Omega}\left( q_k \right) +r-q_k ,
        \end{aligned}
        \nonumber
    \end{equation}
    which can be simplified to
    \begin{equation}
        \label{regularized self-consistency condition for q_k+1}
        q_{k+1}=\gamma P\nabla f_{\Omega}\left( q_k \right) q_{k+1}+\gamma Pe_{\Omega}\left( q_k \right) +r.
    \end{equation}
    Utilizing (\ref{gradient of f_Omega}), $\nabla f_{\Omega}\left( q_k \right)$ is given by
    \begin{equation}
        \label{component in regularized self-consistency condition for q_k+1}
        \begin{aligned}
            &\nabla f_{\Omega}(q_k) \\
            =&\left( \begin{matrix}
                (\nabla \max_{\Omega} \left\{ q_{k}(s_1,\cdot ) \right\} )^\top&		&		\\
                &		\ddots&		\\
                &		&		(\nabla \max_{\Omega} \left\{ q_{k}(s_n,\cdot ) \right\} )^\top\\
                \end{matrix} \right).
        \end{aligned}
    \end{equation}
    Deriving a component-wise formula of (\ref{regularized self-consistency condition for q_k+1}), we have
    \begin{equation}
        \label{component-wise regularized self-consistency condition from NR}
        \begin{aligned}
            &q_{k+1}\left( s_i,a_j \right) =r\left( s_i,a_j \right) \\&+\gamma\sum_{s^{\prime}\in \mathcal{S}}{p}\left( s^{\prime}|s_i,a_j \right) \langle \nabla {\max}_{\Omega} \left\{ q_k(s^{\prime},\cdot ) \right\} ,q_{k+1}\left( s^{\prime},\cdot \right) \rangle\\
            &-\gamma\sum_{s^{\prime}\in \mathcal{S}}{p}\left( s^{\prime}| s_i,a_j \right) \lambda\Omega \left( \nabla {\max}_{\Omega} \left\{ q_k(s^{\prime},\cdot ) \right\} \right).
        \end{aligned}
    \end{equation}
    Utilizing (\ref{gradient of smoothed max operator}), $\nabla \max_{\Omega} \left\{ q_{k}(s_i,\cdot ) \right\}$ is given by
    \begin{equation}
        \label{gradient of smoothed max operator for q_k}
        \begin{aligned}
            &\nabla {\max} _{\Omega}(q_k\left( s_i,\cdot \right) )\\&=\underset{\pi \left( \cdot |s_i \right) \in \Delta _m}{\mathrm{arg}\max}\left\{ \langle \pi \left( \cdot |s_i \right) ,q_k\left( s_i,\cdot \right) \rangle -\lambda\Omega (\pi \left( \cdot |s_i \right) ) \right\}.
        \end{aligned}
    \end{equation}

    On the other side, for RPI, since $\pi_{k+1}=\underset{\pi}{\mathrm{argmax}}\left\{ T^{\pi}_{\Omega}\left( q_{\pi_k} \right) \right\}$, using the definition of regularized Bellman operator (\ref{regularized Bellman operator}), we have
    \begin{equation}
        \label{greedy policy in R-PI}
        \begin{aligned}
            &\pi_{k+1}(\cdot|s_i)\\&=\underset{\pi \left( \cdot |s_i \right) \in \Delta _m}{\mathrm{arg}\max}\left\{ \langle \pi \left( \cdot |s_i \right) ,q_{\pi_k}\left( s_i,\cdot \right) \rangle -\lambda\Omega (\pi \left( \cdot |s_i \right) ) \right\}.
        \end{aligned}
    \end{equation}
    Since $T^{\pi_{k+1}}_{\Omega}$ is a contraction mapping, we have
    \begin{equation}
        \nonumber
        q_{\pi_{k+1}}=\left(T^{\pi_{k+1}}_{\Omega}\right)^{\infty} (q_{\pi_{k}})=T^{\pi_{k+1}}_{\Omega}(q_{\pi_{k+1}}).
    \end{equation}
    Turn it into a component-wise formula using the definition of regularized self-consistency operator (\ref{regularized self-consistency operator}), and we have
    \begin{equation}
        \label{component-wise regularized self-consistency condition from R-PI}
        \begin{aligned}
            &q_{\pi_{k+1}}\left( s_i,a_j \right) =r\left( s_i,a_j \right)\\
            &+\gamma \sum_{s^{\prime}\in \mathcal{S}}{p}\left( s^{\prime}|s_i,a_j \right) \langle \pi _{k+1}\left( \cdot |s^{\prime} \right) ,q_{\pi_{k+1}}\left( s^{\prime},\cdot \right) \rangle\\
            &-\gamma \sum_{s^{\prime}\in \mathcal{S}}{p}\left( s^{\prime}| s_i,a_j \right) \lambda\Omega \left( \pi _{k+1}\left( \cdot |s^{\prime} \right) \right) .\\
        \end{aligned}
    \end{equation}
    Combining (\ref{component-wise regularized self-consistency condition from NR}), (\ref{gradient of smoothed max operator for q_k}), (\ref{greedy policy in R-PI}) and (\ref{component-wise regularized self-consistency condition from R-PI}), we obtain
    \begin{equation}
        \nonumber
        q_{k+1}\left( s_i,a_j \right)=q_{\pi_{k+1}}\left( s_i,a_j \right),
    \end{equation}
    and the proof is completed.
\end{proof}

\subsection{Equivalence between RMPI and Inexact Newton method}

In RMPI, the regularized self-consistency condition is solved approximately in policy evaluations. Since we have proven that regularized self-consistency condition is equivalent to the Newton iteration formula, it is natural to connect RMPI with inexact Newton method in which the Newton iteration formula is solved approximately.

The following proposition provides the analytical result from finite operator steps in policy evaluations of RMPI. The key method is also exploiting the decomposition of smoothed Bellman equation with its Jacobian.

\begin{proposition}[inexact policy evaluation]
    \label{finite-step PEV}
    Suppose that $q_{k+1}$ is obtained by carrying out one-step RMPI from $q_k$ as is described in (\ref{MR-PI}). For $M\geq 1$, we have
    \begin{equation}
        \label{result of finite-step PEV}
        \left( T_{\Omega}^{\pi _{k+1}} \right) ^M\left( q_k \right) =q_k+\sum_{i=0}^{M-1}{\left( \gamma P\nabla f_{\Omega}\left( q_k \right) \right) ^i}F_{\Omega}\left( q_k \right).
    \end{equation}
\end{proposition}

\begin{proof}
    In the proof of Theorem \ref{R-PI as standard NR method}, we have shown that (\ref{regularized self-consistency condition for q_k+1}) is equivalent to regularized self-consistency condition (\ref{component-wise regularized self-consistency condition from R-PI}), so the relationship
    \begin{equation}
        T_{\Omega}^{\pi_{k+1}}(q)=\gamma P \nabla f_{\Omega}\left(q_k\right) q+\gamma P e_{\Omega}\left(q_k\right)+r
        \nonumber
    \end{equation}
    holds.
    Utilizing the decomposition of smoothed Bellman equation at $q_k$ (\ref{decomposition of smoothed Bellman equation at q_k}), we also have
    \begin{equation}
        \label{equivalent form of regularized self-consistency operator}
        T_{\Omega}^{\pi _{k+1}}\left( { q} \right) =\gamma P\nabla f_{\Omega}\left( q_k \right) { q}+F_{\Omega}(q_k)+q_k-\gamma P\nabla f_{\Omega}\left( q_k \right) q_k .
    \end{equation}
    We prove (\ref{result of finite-step PEV}) by induction. For $M=1$, $T_{\Omega}^{\pi_{k+1}}(q_k)=q_k+F_{\Omega}\left( q_k \right)$ holds, which can be obtained by substituting $q=q_k$ into (\ref{equivalent form of regularized self-consistency operator}). So (\ref{result of finite-step PEV}) is satisfied for $M=1$. Suppose that (\ref{result of finite-step PEV}) holds for $M=j$. We then have
    \begin{equation}
        \begin{aligned}
            &\left( T_{\Omega}^{\pi _{k+1}} \right) ^{j+1}\left( q_k \right)\\
            =&T_{\Omega}^{\pi _{k+1}}\left( \left( T_{\Omega}^{\pi _{k+1}} \right) ^j\left( q_k \right) \right)\\
            =&T_{\Omega}^{\pi _{k+1}}\left( q_k+\sum_{i=0}^{j-1}{\left( \gamma P\nabla f_{\Omega}\left( q_k \right) \right) ^i}F_{\Omega}\left( q_k \right) \right)\\
            =&\gamma P\nabla f_{\Omega}\left( q_k \right) q_k+\sum_{i=1}^j{\left( \gamma P\nabla f_{\Omega}\left( q_k \right) \right) ^i}F_{\Omega}\left( q_k \right)\\
            &+F_{\Omega}\left( q_k \right) +q_k-\gamma P\nabla f_{\Omega}\left( q_k \right) q_k\\
            =&q_k+\sum_{i=0}^j{\left( \gamma P\nabla f_{\Omega}\left( q_k \right) \right) ^i}F_{\Omega}\left( q_k \right) ,
        \end{aligned}
        \nonumber
    \end{equation}
    which indicates that (\ref{result of finite-step PEV}) holds for $M=j+1$. This completes the proof.
\end{proof}

With the aid of Proposition \ref{finite-step PEV}, the equivalence between RMPI and inexact Newton method is established in the following theorem. We identify the connection between the error incurred by inexact policy evaluation and the error in solving the Newton iteration formula.

\begin{theorem}[equivalence between RMPI and inexact Newton method]
    \label{MR-PI as I-NR method}
    Suppose that $q_{k+1}$ is obtained by carrying out one-step RMPI from $q_k$ as is described in (\ref{MR-PI}). Then $q_{k+1}$ satisfies the following inexact Newton iteration:
    \begin{equation}
        \label{I-NR representation of MR-PI}
        \left\{ \begin{array}{l}
            F_{\Omega}^{\prime}(q_k)s_k=-F_{\Omega}(q_k)+r_k\\
            q_{k+1}=q_k+s_k\\
        \end{array} \right. \quad,
    \end{equation}
    in which $r_k$ is given by
    \begin{equation}
        \label{error term of MR-PI}
        r_k=\left( \gamma P\nabla f_{\Omega}\left( q_k \right) \right) ^MF_{\Omega}\left( q_k \right) .
    \end{equation}
\end{theorem}

\begin{proof}
    We prove this theorem by directly calculating $r_k$. Utilizing (\ref{I-NR representation of MR-PI}), we have
    \begin{equation}
        \left( \gamma P\nabla f_{\Omega}\left( q_k \right) -I \right) \left( q_{k+1}-q_k \right) =-F_{\Omega}\left( q_k \right) +r_k.
        \nonumber
    \end{equation}
    Using Proposition \ref{finite-step PEV}, we obtain
    \begin{equation}
        \begin{aligned}
            &r_k=\left( \gamma P\nabla f_{\Omega}\left( q_k \right) -I \right) \left( q_{k+1}-q_k \right) +F_{\Omega}\left( q_k \right)\\=&F_{\Omega}\left( q_k \right) +\left( \gamma P\nabla f_{\Omega}\left( q_k \right) -I \right) \sum_{i=0}^{M-1}{\left( \gamma P\nabla f_{\Omega}\left( q_k \right) \right) ^i}F_{\Omega}\left( q_k \right)\\
            =&F_{\Omega}\left( q_k \right) \\&+\left( \sum_{i=1}^M{\left( \gamma P\nabla f_{\Omega}\left( q_k \right) \right) ^i}-\sum_{i=0}^{M-1}{\left( \gamma P\nabla f_{\Omega}\left( q_k \right) \right) ^i} \right) F_{\Omega}\left( q_k \right)\\
            =&F_{\Omega}\left( q_k \right) +\left( \left( \gamma P\nabla f_{\Omega}\left( q_k \right) \right) ^M-I \right) F_{\Omega}\left( q_k \right)\\
            =&\left( \gamma P\nabla f_{\Omega}\left( q_k \right) \right) ^MF_{\Omega}\left( q_k \right).
        \end{aligned}
        \nonumber
    \end{equation}
\end{proof}

\section{Convergence Analysis}

In this section, we establish the convergence results for RPI and RMPI with the aid of their connections to NR method.

\subsection{Local Quadratic Convergence of RPI}

To prove the local quadratic convergence of RPI, the following lemma is needed.

\begin{lemma}
    \label{lemma for quadratic bound}
    Suppose $F:D \subset \mathbb{R}^n \rightarrow \mathbb{R}^n$ is differentiable on the convex set $D$ and there exists a constant $L$ satisfies
    \begin{equation}
        \Vert F^{\prime}\left( x \right) -F^{\prime}\left( y \right) \Vert \leq L \Vert x-y\Vert ,\quad \forall x,y\in D.
        \nonumber
    \end{equation}
    Then we have
    \begin{equation}
        \left\| F(x)-F(y)-F^{\prime}(y)(x-y) \right\| \le \frac{L}{2}\Vert x-y \Vert ^2,\quad \forall x,y\in D.
        \nonumber
    \end{equation}
\end{lemma}
\begin{proof}
    See \cite{ortega2000iterative}.
\end{proof}

Using Lemma \ref{lemma for quadratic bound}, along with Proposition \ref{properties of the Jacobian's inverse} and Proposition \ref{Lipschitz continuity of the Jacobian}, we can bound the errors between Newton iterations and derive the following quadratic convergence result for RPI. Recall that $\lambda$ denotes the regularization parameter, $\mu$ denotes that the regularizer is $\mu$-strongly convex with respect to $\Vert \cdot \Vert_1$, $n$ denotes the dimension of state space and $m$ denotes the dimension of action space.

\begin{theorem}[local quadratic convergence of RPI]
    \label{local quadratic convergence of regularized PI}
    RPI (\ref{R-PI}) is quadratically convergent in the region
    \begin{equation}
        \label{quadratic convergence region}
        \left\| q_*-q \right\| _{\infty} \le \frac{2}{3}\frac{1-\gamma}{\gamma}\lambda\mu.
    \end{equation}
    The iteration error satisfies
    \begin{equation}
        \label{quadratic convergence between iterations}
        \left\| q_*-q_{k+1} \right\| _{\infty}\le \frac{3}{2}\frac{\gamma}{1-\gamma}\frac{1}{\lambda\mu}\left\| q_*-q_k \right\| _{\infty}^{2}
    \end{equation}
    and
    \begin{equation}
        \label{total quadratic convergence bound}
        \begin{aligned}
            \left\| q_*-q_{k} \right\| _{\infty} \leq \frac{2}{3}\frac{1-\gamma}{\gamma}\lambda\mu  \cdot   \left(\frac{3}{2}\frac{\gamma}{1-\gamma}\frac{1}{\lambda\mu}\left\| q_*-q_0 \right\| _{\infty}\right)^{2^k}.
        \end{aligned}
    \end{equation}
\end{theorem}

\begin{proof}
    Utilizing Proposition \ref{Lipschitz continuity of the Jacobian}, we have
    \begin{equation}
        \Vert F^{\prime}_{\Omega}(q_k)-F^{\prime}_{\Omega}(q_*)\Vert _{\infty}\leq  \frac{\gamma}{\lambda\mu}\Vert q_k-q_*\Vert _{\infty}.
        \nonumber
    \end{equation}
    Utilizing Proposition \ref{properties of the Jacobian's inverse}, the Jacobian's inverse is bounded by
    \begin{equation}
        \left\| F_{\Omega}^{\prime}\left( q_k \right) ^{-1} \right\|_{\infty} \leq \frac{1}{1-\gamma}.
        \nonumber
    \end{equation}
    Utilizing Lemma \ref{lemma for quadratic bound} and Proposition \ref{Lipschitz continuity of the Jacobian}, we have
    \begin{equation}
        \begin{aligned}
            &\left\| F_{\Omega}\left( q_k \right) -F_{\Omega}\left( q_* \right) -F_{\Omega}^{\prime}\left( q_* \right) \left( q_k-q_* \right) \right\| _{\infty} \\ &\leq  \frac{\gamma}{2 \lambda \mu} \left\| q_k-q_* \right\| _{\infty}^{2}.
        \end{aligned}
        \nonumber
    \end{equation}
    Putting these bounds together, we obtain
    \begin{equation}
        \begin{aligned}
            &\left\| q_{k+1}-q_* \right\| _{\infty}=\left\| q_k-F_{\Omega}^{\prime}\left( q_k \right) ^{-1}F_{\Omega}\left( q_k \right) -q_* \right\| _{\infty}\\
            =&\Vert -F_{\Omega}^{\prime}\left( q_k \right) ^{-1}\left( F_{\Omega}\left( q_k \right) -F_{\Omega}\left( q_* \right) -F_{\Omega}^{\prime}\left( q_* \right) \left( q_k-q_* \right) \right)\\
            &+F_{\Omega}^{\prime}\left( q_k \right) ^{-1}\left( F_{\Omega}^{\prime}\left( q_k \right) -F_{\Omega}^{\prime}\left( q_* \right) \right) \left( q_k-q_* \right) \Vert _{\infty}\\
            \le &\left\| F_{\Omega}^{\prime}\left( q_k \right) ^{-1} \right\| _{\infty}\left\| F_{\Omega}\left( q_k \right) -F_{\Omega}\left( q_* \right) -F_{\Omega}^{\prime}\left( q_* \right) \left( q_k-q_* \right) \right\| _{\infty}\\
            &+\left\| F_{\Omega}^{\prime}\left( q_k \right) ^{-1} \right\| _{\infty}\left\| F_{\Omega}^{\prime}\left( q_k \right) -F_{\Omega}^{\prime}\left( q_* \right) \right\| _{\infty}\left\| q_k-q_* \right\| _{\infty}\\
            \le &\left\| F_{\Omega}^{\prime}\left( q_k \right) ^{-1} \right\| _{\infty} \frac{\gamma}{2\lambda\mu}\left\| q_k-q_* \right\| _{\infty}^{2}\\
            &+\left\| F_{\Omega}^{\prime}\left( q_k \right) ^{-1} \right\| _{\infty}\frac{\gamma}{\lambda\mu}\left\| q_k-q_* \right\| _{\infty}^{2}\\
            \le &\frac{3}{2}\frac{\gamma}{1-\gamma}\frac{1}{\lambda\mu}\left\| q_k-q_* \right\| _{\infty}^{2},
        \end{aligned}
        \nonumber
    \end{equation}
    which proves (\ref{quadratic convergence between iterations}).
    Although (\ref{quadratic convergence between iterations}) holds for all $k\geq 0$, it is not meaningful until it leads to error reduction:
    \begin{equation}
        \frac{3}{2}\frac{\gamma}{1-\gamma}\frac{1}{\lambda\mu}\left\| q_*-q_k \right\| _{\infty}^{2}\leq \left\| q_*-q_k \right\| _{\infty},
        \nonumber
    \end{equation}
    which provides the quadratic convergence region described by (\ref{quadratic convergence region}). (\ref{total quadratic convergence bound}) is a direct result of (\ref{quadratic convergence between iterations}).
\end{proof}

\begin{remark}
In Theorem~\ref{local quadratic convergence of regularized PI}, we assume that the regularizer $\Omega$ is $\mu$-strongly convex with respect to $\|\cdot\|_1$. In the case of (minus) Shannon entropy, we have $\Omega(\pi(\cdot\mid s))=\sum_a \pi(a\mid s)\log \pi(a\mid s)$ and $\mu=1$. Thus, for any strongly convex regularizer whose constant $\mu$ is dimension-independent, both the local quadratic convergence neighborhood and the associated rate are dimension-free, with Shannon entropy as the most prominent example.
For (minus) second-order Tsallis entropy, $\Omega(\pi(\cdot \mid s))=\frac{1}{2}\left(\sum_a \pi(a \mid s)^2-1\right)$, which is $1$-strongly convex with respect to $\|\cdot\|_2$. By invoking the norm bound $\|x\|_1 \le \sqrt{m}\|x\|_2$, where $m=|\mathcal{A}|$, the corresponding strong convexity constant with respect to $\|\cdot\|_1$ becomes $\mu=1/m$. Consequently, the local quadratic region deteriorates with the action dimension $m$. This may partly explain why Shannon entropy regularization \cite{haarnoja2018soft, duan2021distributional} is more prevalent than Tsallis variants \cite{lee2018sparse, chow2018path} in state-of-the-art RL algorithms.
\end{remark}

\textbf{Comparison with Prior Art.}
To the best of our knowledge, this paper is the first to establish local quadratic convergence of RPI for general strongly convex regularizers. Existing analyses of RPI (e.g., \cite{geist2019theory}) prove global linear convergence but do not characterize its local superlinear behavior.
A closely related result is derived by Cen et al.~\cite{cen2022fast} for soft policy iteration, i.e., RPI with the specific choice of Shannon entropy as regularizer. They establish an asymptotic quadratic convergence bound, which can be restated in our notation as,
\begin{equation}
\label{previous quadratic convergence result}
\langle \rho_*, v_* - v_{k+1}\rangle
\le
\frac{4\gamma^2}{(1-\gamma)\lambda}\,\frac{1}{\rho_{*,\min}}
\langle \rho_*, v_* - v_{k}\rangle^2,
\end{equation}
where $v_*$ is the optimal regularized state-value function, $v_k$ is the state-value iterate, $\rho_*\in\Delta_{|\mathcal{S}|}$ is the stationary state distribution induced by the optimal policy, and $\rho_{*,\min}= \min_{s\in\mathcal{S}}\rho_*(s)$.
Our local quadratic guarantee in \eqref{quadratic convergence between iterations} is more general, as it applies to any $\mu$-strongly convex regularizer $\Omega$, whereas \cite{cen2022fast} is restricted to Shannon entropy. Furthermore, even when specialized to the Shannon entropy case, our guarantee is stronger in the following three aspects.
\begin{enumerate}
\item \emph{Quantity controlled.}
Our bound \eqref{quadratic convergence between iterations} controls the infinity norm error of the action-value function, providing uniform (worst-case) control over all state-action pairs.
This is strictly stronger than the scalar performance metric $\langle \rho_*, v_* - v_k\rangle$, which only measures the value error on average under $\rho_*$.
In particular, since the regularized state-value mapping $v=f_\Omega(q)$ is $1$-Lipschitz under $\|\cdot\|_\infty$, we have $\|v_* - v_k\|_\infty \le \|q_* - q_k\|_\infty$ and $\langle \rho_*, v_* - v_k\rangle \le \|v_* - v_k\|_\infty$. Therefore, our result immediately implies a distribution-weighted value-gap bound of the type in \eqref{previous quadratic convergence result}, whereas the converse does not hold.
\item \emph{Constants derived.}
For Shannon entropy, $\mu=1$ (with respect to $\|\cdot\|_1$ on the probability simplex), and \eqref{quadratic convergence between iterations} specializes to
\[
\|q_* - q_{k+1}\|_\infty
\;\le\;
\frac{3}{2}\frac{\gamma}{1-\gamma}\frac{1}{\lambda}\,
\|q_* - q_k\|_\infty^2,
\]
whose coefficient depends only on $(\gamma,\lambda)$ and is independent of $|\mathcal{S}|$ and $|\mathcal{A}|$.
Consequently, the associated explicit quadratic convergence region
$
\|q_* - q\|_\infty \le \frac{2}{3}\frac{1-\gamma}{\gamma}\lambda
$
is also dimension-free.
In contrast, the prefactor in \eqref{previous quadratic convergence result} scales with $1/\rho_{*,\min}$. Since $\sum \rho_*(s) = 1$, the term $1/\rho_{*,\min}$ is at least $|\mathcal{S}|$, introducing an unavoidable dependence on the state space dimension. Furthermore, this term can be arbitrarily large if the optimal policy visits certain states with very low probability, and the bound becomes vacuous if $\rho_{*,\min} \to 0$. As a result, the quadratic convergence region implied by~\eqref{previous quadratic convergence result},
\begin{equation}
    \label{required region}
    \langle \rho_*, v_* - v\rangle \;\le\; \frac{(1-\gamma)\lambda}{4\gamma^2}\,\rho_{*,\min},
\end{equation}
shrinks as the dimension increases and may effectively vanish in ill-conditioned problems.
\item \emph{Assumptions needed.}
Beyond the local-region requirement \eqref{required region}, the bound in \eqref{previous quadratic convergence result} also hinges on additional nontrivial local assumptions, including: (i) a nondegeneracy condition $\rho_{*,\min}>0$ and (ii) a policy-space proximity condition $\|\log\pi_k-\log\pi_*\|_\infty\le 1$.
In contrast, our local quadratic guarantee imposes no additional distributional or policy-level conditions beyond the explicit neighborhood requirement
$
\|q_* - q\|_\infty \le \frac{2}{3}\frac{1-\gamma}{\gamma}\lambda.
$
\end{enumerate}

\textbf{Comparison with Unregularized Case.}
In standard MDPs without regularization, the Bellman operator is nonsmooth due to the max operation. While PI can still be interpreted as a Newton method with stepwise linearization of the nonsmooth operator \cite{puterman1979convergence, bertsekas1997nonlinear}, it is hard to establish a local quadratic convergence guarantee without additional assumptions. Gargiani et al. \cite{gargiani2022dynamic} show that PI is locally quadratically convergent under a nondegeneracy assumption excluding spurious greedy policies. Consequently, quadratic convergence of PI has two caveats: (1) it applies only to a subset of MDPs where such degeneracies do not occur, and (2) even when a quadratic neighborhood exists, its radius is problem-dependent and can be arbitrarily small (e.g., under near-ties of multiple actions). In contrast, RPI enjoys an unconditional local quadratic convergence guarantee with an explicit neighborhood, the size of which will not vanish for any problem. We stress that this is not a strict apples-to-apples comparison: standard PI converges to the unregularized optimum, whereas RPI converges to the regularized optimum depending on $\lambda$ and $\Omega$. Nonetheless, regularization provides a robust Newton-like regime with an explicit, stable neighborhood that does not hinge on delicate nondegeneracy conditions.

\subsection{Local Fast Linear Convergence for RMPI}

In this section, we show that RMPI achieves faster linear convergence as it approaches the optimal value, with the asymptotic convergence rate being $\gamma^M$ in which $M$ denotes the number of operator steps carried out in policy evaluation. First we need the following two lemmas.

\begin{lemma}
    \label{lemma A for local convergence of MR-PI}
    Given $\Delta >0$, $\eta>0$ and $0<\gamma<1$, define
    \begin{equation}
        \label{delta for scaling}
        \delta=(1-\gamma)\left(\sqrt{1+\frac{\Delta}{\eta+2}}-1\right).
    \end{equation}
    Then we have
    \begin{equation}
        \label{lemma for scaling}
        \left(1+\frac{\delta}{1-\gamma}\right)\left(\eta+\frac{\eta+2}{1-\gamma} \delta\right)<\eta+\Delta.
    \end{equation}
\end{lemma}
\begin{proof}
    Define $t=1+\frac{\delta}{1-\gamma}$ and we have
    \begin{equation}
        \begin{aligned}
            \Delta & =(\eta+2)\left(t^2-1\right) \\
            & =((\eta+2) t+\eta+2)(t-1) \\
            & >((\eta+2) t+\eta)(t-1).
        \end{aligned}
        \nonumber
    \end{equation}
    Rearranging terms, we obtain
    \begin{equation}
        \begin{aligned}
            \eta+\Delta & >((\eta+2) t+\eta) t-(\eta+2) t \\
            & =t(\eta t+2 t-2) \\
            & =\left(1+\frac{\delta}{1-\gamma}\right)\left(\eta+\frac{\delta}{1-\gamma} \eta+\frac{2 \delta}{1-\gamma}\right),
        \end{aligned}
        \nonumber
    \end{equation}
    which completes the proof.
\end{proof}

\begin{lemma}
    \label{lemma B for local convergence of MR-PI}
    Given $\delta>0$, define
    \begin{equation}
        \label{epsilon for scaling}
        \varepsilon=\frac{(1-\gamma)^3}{(1+\gamma) \gamma} \lambda\mu\delta.
    \end{equation}
    Then the following bounds hold when $\Vert q-q_*\Vert_{\infty}\leq\frac{1+\gamma}{1-\gamma} \varepsilon $:
    \begin{equation}
        \label{bound 1 for scaling}
        \left\| F_{\Omega}^{\prime}(q)-F_{\Omega}^{\prime}\left( q_* \right) \right\| _{\infty}\le \delta ,
    \end{equation}
    \begin{equation}
        \label{bound 2 for scaling}
        \left\| F_{\Omega}^{\prime}(q)^{-1}-F_{\Omega}^{\prime}\left( q_* \right) ^{-1} \right\| _{\infty}\le \delta,
    \end{equation}
    \begin{equation}
        \label{bound 3 for scaling}
        \left\| F_{\Omega}(q)-F_{\Omega}\left( q_* \right) -F_{\Omega}^{\prime}\left( q_* \right) \left( q-q_* \right) \right\| _{\infty}\le \delta \Vert q-q_*\Vert_{\infty}.
    \end{equation}
\end{lemma}

\begin{proof}
    The bound in (\ref{bound 1 for scaling}) can be identified using the Lipschitz continuity of the Jacobian (\ref{expression of Lipschitz continuity of the Jacobian}) stated in Proposition \ref{Lipschitz continuity of the Jacobian}. The bound in (\ref{bound 3 for scaling}) can be identified using Lemma \ref{lemma for quadratic bound}. Besides, we have
    \begin{equation}
        \begin{aligned}
            &\left\| F_{\Omega}^{\prime}(y)^{-1}-F_{\Omega}^{\prime}(x)^{-1} \right\| _{\infty}\\
            =&\left\| F_{\Omega}^{\prime}(x)^{-1}\left( F_{\Omega}^{\prime}(x)-F_{\Omega}^{\prime}(y) \right) F_{\Omega}^{\prime}(y)^{-1} \right\| _{\infty}\\
            \leq &\left\| F_{\Omega}^{\prime}(x)^{-1} \right\| _{\infty}\left\| F_{\Omega}^{\prime}(x)-F_{\Omega}^{\prime}(y) \right\| _{\infty}\left\| F_{\Omega}^{\prime}(y)^{-1} \right\| _{\infty}\\
            \leq &\frac{\gamma}{(1-\gamma )^2\lambda\mu}\left\| x-y\right\| _{\infty},
        \end{aligned}
        \nonumber
    \end{equation}
    which can be used to check the bound in (\ref{bound 2 for scaling}).
\end{proof}

Utilizing Lemma \ref{lemma A for local convergence of MR-PI} and Lemma \ref{lemma B for local convergence of MR-PI}, we show that RMPI achieves an asymptotic convergence rate of $\gamma^M$ in the following theorem.

\begin{theorem}[local fast linear convergence of RMPI]
    \label{local fast linear convergence of RMPI}
    Suppose $M\geq2$.
    Given $0<\Delta<\gamma-\gamma^M$, suppose $q_0$ satisfies
    \begin{equation}
        \label{local fast linear convergence region}
        \Vert q_0 - q_* \Vert_{\infty}<\varepsilon,
    \end{equation}
    in which $\varepsilon$ is defined in (\ref{epsilon for scaling}) and $\delta$ is defined in (\ref{delta for scaling}) with $\eta=\gamma^M$.
    Then the sequence $\left\{q_k \right\}$ generated by RMPI (\ref{MR-PI}) satisfies
    \begin{equation}
        \label{local fast linear convergence rate of MR-PI}
        \left\| q_{k+1}-q_* \right\| _{q_*}\le \left( \gamma ^M+\Delta \right) \left\| q_{k}-q_* \right\| _{q_*},
    \end{equation}
    in which $\Vert \cdot\Vert_{q_*}$ denotes the norm defined by
    \begin{equation}
        \nonumber
        \Vert q\Vert_{q_*}=\Vert F_{\Omega}^{\prime}\left( q_* \right)q\Vert_{\infty}.
    \end{equation}
    Furthermore, we have
    \begin{equation}
        \label{total fast linear convergence bound}
        \begin{aligned}
            &\left\| q_*-q_{k} \right\| _{\infty}\\\le& \frac{1}{1-\gamma}\frac{1+\frac{\delta}{1-\gamma}}{1-\frac{\delta}{1-\gamma}}\left( \gamma ^M+\Delta \right) ^k \\ & \cdot \min \left\{ \left\| F_{\Omega}\left(q_0\right) \right\| _{\infty},(1+\gamma )\left( 1-\frac{\delta}{1-\gamma} \right) \left\| q_0-q_* \right\| _{\infty} \right\},
            \end{aligned}
    \end{equation}
    and the asymptotic linear convergence rate of RMPI is $\gamma^M$.
\end{theorem}

\begin{proof}
    Since we have proven that RMPI is equivalent to inexact Newton method in Theorem \ref{MR-PI as I-NR method}, we use (\ref{I-NR representation of MR-PI}) as the representation of RMPI. First, we prove the norm $\Vert \cdot\Vert_{*}$ is equivalent to infinity norm $\Vert \cdot\Vert_{\infty}$. Utilizing the boundedness of the Jacobian and its inverse, we have
    \begin{equation}
        \|q\|_{q_*}=\left\|F_{\Omega}^{\prime}\left(q_*\right) q\right\|_{\infty} \leq\left\|F_{\Omega}^{\prime}\left(q_*\right)\right\|_{\infty}\|q\|_{\infty} \leq(1+\gamma)\|q\|_{\infty}
        \nonumber
    \end{equation}
    and
    \begin{equation}
        \begin{aligned}
            \|q\|_{\infty} & =\left\|F_{\Omega}^{\prime}\left(q_*\right)^{-1} F_{\Omega}^{\prime}\left(q_*\right) q\right\|_{\infty} \leq\left\|F_{\Omega}^{\prime}\left(q_*\right)^{-1}\right\|_{\infty}\|q\|_{q_*} \\
            & \leq \frac{1}{1-\gamma}\|q\|_{q_*} ,
        \end{aligned}
        \nonumber
    \end{equation}
    which induces
    \begin{equation}
        \label{equivalent norm relationship}
        (1-\gamma)\|q\|_{\infty} \leq\|q\|_{q_*} \leq(1+\gamma)\left\|q\right\|_{\infty}.
    \end{equation}

    Next, we prove the following argument holds for $n\geq1$ using mathematical induction:
    \begin{equation}
        \label{argument for mathematical induction in MR-PI}
        \begin{aligned}
            &\left\| q_n-q_* \right\| _{\infty}\le \frac{1+\gamma}{1-\gamma} \varepsilon,\\
            &\left\| q_n-q_* \right\| _{q_*}\le \left( \gamma ^M+\Delta \right) \left\| q_{n-1}-q_* \right\| _{q_*} .
        \end{aligned}
    \end{equation}
    Suppose (\ref{argument for mathematical induction in MR-PI}) holds for $n=1,2,\cdots,k$. We need to show that it also holds for $n=k+1$.
    Recall that the error term $r_k$ is given in (\ref{I-NR representation of MR-PI}) and (\ref{error term of MR-PI}). $\left\|q_{k+1}-q_*\right\|_*$ is bounded by
    \begin{equation}
        \begin{aligned}
            & \left\|q_{k+1}-q_*\right\|_{q_*} \\
            = & \left\|F_{\Omega}^{\prime}\left(q_*\right)\left(q_{k+1}-q_*\right)\right\|_{\infty} \\
            = & \left\|F_{\Omega}^{\prime}\left(q_*\right) F_{\Omega}^{\prime}\left(q_k\right)^{-1} F_{\Omega}^{\prime}\left(q_k\right)\left(q_{k+1}-q_*\right)\right\|_{\infty} \\
            \leq & \left\|A\right\|_{\infty} \left\|B\right\|_{\infty},
        \end{aligned}
        \nonumber
    \end{equation}
    in which $A$ is given by
    \begin{equation}
        \begin{aligned}
            A & = F_{\Omega}^{\prime}\left(q_*\right) F_{\Omega}^{\prime}\left(q_k\right)^{-1}\\
            &=F_{\Omega}^{\prime}\left(q_*\right) F_{\Omega}^{\prime}\left(q_k\right)^{-1}+I-F_{\Omega}^{\prime}\left(q_*\right) F_{\Omega}^{\prime}\left(q_*\right)^{-1} \\
            & =I+F_{\Omega}^{\prime}\left(q_*\right)\left(F_{\Omega}^{\prime}\left(q_k\right)^{-1}-F_{\Omega}^{\prime}\left(q_*\right)^{-1}\right)
        \end{aligned}
        \nonumber
    \end{equation}
    and $B$ is given by
    \begin{equation}
        \begin{aligned}
            B=&F_{\Omega}^{\prime}\left( q_k \right) \left( q_{k+1}-q_* \right)\\
            =&r_k+\left( F_{\Omega}^{\prime}\left( q_k \right) -F_{\Omega}^{\prime}\left( q_* \right) \right) \left( q_k-q_* \right)\\
            &-\left( F_{\Omega}\left( q_k \right) -F_{\Omega}\left( q_* \right) -F_{\Omega}^{\prime}\left( q_* \right) \left( q_k-q_* \right) \right).
        \end{aligned}
        \nonumber
    \end{equation}
    Utilizing Lemma \ref{lemma B for local convergence of MR-PI}, $\|A\|_{\infty}$ and $\|B\|_{\infty}$ can be bounded as
    \begin{equation}
        \begin{aligned}
            \|A\|_{\infty} & \leq 1+\left\|F_{\Omega}^{\prime}\left(q_*\right)\right\|_{\infty}\left\|F_{\Omega}^{\prime}\left(q_k\right)^{-1}-F_{\Omega}^{\prime}\left(q_*\right)^{-1}\right\|_{\infty} \\
            & \leq 1+(1+\gamma) \delta
        \end{aligned}
        \nonumber
    \end{equation}
    and
    \begin{equation}
        \begin{aligned}
            \|B\|_{\infty} \leq & \left\|r_k\right\|_{\infty}+\left\|F_{\Omega}^{\prime}\left(q_k\right)-F_{\Omega}^{\prime}\left(q_*\right)\right\|_{\infty}\left\|q_k-q_*\right\|_{\infty} \\
            & +\left\|F_{\Omega}\left(q_k\right)-F_{\Omega}\left(q_*\right)-F_{\Omega}^{\prime}\left(q_*\right)\left(q_k-q_*\right)\right\|_{\infty} \\
            \leq & \gamma^M\left\|F_{\Omega}\left(q_k\right)\right\|_{\infty}+2 \delta\left\|q_k-q_*\right\|_{\infty}.
        \end{aligned}
        \nonumber
    \end{equation}
    $\left\| F_{\Omega}\left( q_k \right) \right\| _{\infty}$ is bounded by
    \begin{equation}
        \begin{aligned}
            &\left\| F_{\Omega}\left( q_k \right) \right\| _{\infty}\\
            =&\Vert F_{\Omega}^{\prime}\left( q_* \right) \left( q_k-q_* \right) \\&+\left( F_{\Omega}\left( q_k \right) -F_{\Omega}\left( q_* \right) -F_{\Omega}^{\prime}\left( q_* \right) \left( q_k-q_* \right) \right) \Vert _{\infty}\\
            \leq &\left\| q_k-q_* \right\| _{q_*}+\left\| F_{\Omega}\left( q_k \right) -F_{\Omega}\left( q_* \right) -F_{\Omega}^{\prime}\left( q_* \right) \left( q_k-q_* \right) \right\| _{\infty}\\
            \leq &\left\| q_k-q_* \right\| _{q_*}+\delta \left\| q_k-q_* \right\| _{\infty}.
        \end{aligned}
        \nonumber
    \end{equation}
    Putting above bounds together and using Lemma \ref{lemma A for local convergence of MR-PI}, we obtain
    \begin{equation}
        \begin{aligned}
            &\left\|q_{k+1}-q_*\right\|_{q_*} \leq\|A\|_{\infty}\|B\|_{\infty} \\
            \leq&  (1+(1+\gamma) \delta)\left(\gamma^M+\left(\gamma^M+2\right) \frac{\delta}{1-\gamma}\right)\left\|q_k-q_*\right\|_{q_*} \\
            <&  (1+\frac{\delta}{1-\gamma})\left(\gamma^M+\left(\gamma^M+2\right) \frac{\delta}{1-\gamma}\right)\left\|q_k-q_*\right\|_{q_*} \\
            <& (\gamma^M + \Delta) \left\|q_k-q_*\right\|_{q_*} .
        \end{aligned}
        \nonumber
    \end{equation}
    We also have
    \begin{equation}
        \begin{aligned}
            \left\|q_{k+1}-q_*\right\|_{\infty} & \leq \frac{1}{1-\gamma}\left\|q_{k+1}-q_*\right\|_{q_*} \\
            & \leq \frac{1}{1-\gamma}\left(\gamma^M+\Delta\right)^{k+1}\left\|q_0-q_*\right\|_{q_*} \\
            & \leq \frac{1+\gamma}{1-\gamma}\left(\gamma^M+\Delta\right)^{k+1}\left\|q_0-q_*\right\|_{\infty} \\
            & \leq \frac{1+\gamma}{1-\gamma} \varepsilon ,
        \end{aligned}
        \nonumber
    \end{equation}
    which completes the proof of (\ref{argument for mathematical induction in MR-PI}).

    To prove (\ref{total fast linear convergence bound}), we need to identify the relationship between $\Vert F_{\Omega}(q)\Vert_{\infty}$ and $\Vert q-q_*\Vert_{q_*}$.
    $\left\|F_{\Omega}(q)\right\|_{\infty}$ is bounded by
    \begin{equation}
        \begin{aligned}
            &\left\|F_{\Omega}(q)\right\|_{\infty} \\=& \left\|F_{\Omega}^{\prime}\left(q_*\right)\left(q-q_*\right)+F_{\Omega}(q)-F_{\Omega}\left(q_*\right)-F_{\Omega}^{\prime}\left(q_*\right)\left(q-q_*\right)\right\|_{\infty} \\
            \leq& \left\|F_{\Omega}^{\prime}\left(q_*\right)\left(q-q_*\right)\right\|_{\infty}\\&+\left\|F_{\Omega}(q)-F_{\Omega}\left(q_*\right)-F_{\Omega}^{\prime}\left(q_*\right)\left(q-q_*\right)\right\|_{\infty} \\
            \leq & \left\|q-q_*\right\|_{q_*}+\delta\left\|q-q_*\right\|_{\infty} \\
            \leq& \left\|q-q_*\right\|_{q_*}+\frac{\delta}{1-\gamma}\left\|q-q_*\right\|_{q_*} \\
            =& \left(1+\frac{\delta}{1-\gamma}\right)\left\|q-q_*\right\|_{q_*}
        \end{aligned}
        \nonumber
    \end{equation}
    and
    \begin{equation}
        \begin{aligned}
            \left\|F_{\Omega}(q)\right\|_{\infty} & \geq\left\|q-q_*\right\|_{q_*}-\delta\left\|q-q_*\right\|_{\infty} \\
            & \geq\left\|q-q_*\right\|_{q_*}-\frac{\delta}{1-\gamma}\left\|q-q_*\right\|_{q_*} \\
            & =\left(1-\frac{\delta}{1-\gamma}\right)\left\|q-q_*\right\|_{q_*} .
        \end{aligned}
        \nonumber
    \end{equation}
    Combining the above two bounds with (\ref{argument for mathematical induction in MR-PI}), we have
    \begin{equation}
        \left\| F_{\Omega}\left( q_k \right) \right\| _{\infty}\leq \frac{1+\frac{\delta}{1-\gamma}}{1-\frac{\delta}{1-\gamma}}\left( \gamma ^M+\Delta \right) ^k\left\| F_{\Omega}\left( q_0 \right) \right\| _{\infty}.
        \nonumber
    \end{equation}
    Utilizing the relationship (\ref{equivalent norm relationship}), we also have
    \begin{equation}
        \begin{aligned}
            &\left\| F_{\Omega}\left( q_k \right) \right\| _{\infty}\leq \left( 1+\frac{\delta}{1-\gamma} \right) \left\| q_k-q_* \right\| _{q_*}\\
            \leq& \left( 1+\frac{\delta}{1-\gamma} \right) \left( \gamma ^M+\Delta \right) ^k\left\| q_0-q_* \right\| _{q_*}\\
            \leq& (1+\gamma )\left( 1+\frac{\delta}{1-\gamma} \right) \left( \gamma ^M+\Delta \right) ^k\left\| q_0-q_* \right\| _{\infty} .
        \end{aligned}
        \nonumber
    \end{equation}
    Now we can bound the error $\left\| q_*-q_{k} \right\| _{\infty}$. We have
    \begin{equation}
        \label{final bound part 2}
        \begin{aligned}
            \left\| q_*-q_k \right\| _{\infty}&\leq \frac{1}{1-\gamma}\left\| F_{\Omega}\left( q_k \right) \right\| _{\infty}\\&\leq \frac{1}{1-\gamma}\frac{1+\frac{\delta}{1-\gamma}}{1-\frac{\delta}{1-\gamma}}\left( \gamma ^M+\Delta \right) ^k\left\| F_{\Omega}\left( q_0 \right) \right\| _{\infty}
        \end{aligned}
    \end{equation}
    and
    \begin{equation}
        \label{final bound part 3}
        \begin{aligned}
            \left\| q_*-q_k \right\| _{\infty}&\leq \frac{1}{1-\gamma}\left\| F_{\Omega}\left( q_k \right) \right\| _{\infty}\\&\leq \frac{1+\gamma}{1-\gamma}\left( 1+\frac{\delta}{1-\gamma} \right) \left( \gamma ^M+\Delta \right) ^k\left\| q_0-q_* \right\| _{\infty}.
        \end{aligned}
    \end{equation}
    Combining (\ref{final bound part 2}) and (\ref{final bound part 3}), the proof of (\ref{total fast linear convergence bound}) is completed.

    Forcing $\Delta$ goes to zero requires that $q$ approaches the optimal value $q_*$, i.e., for $q_0\rightarrow q_*$,
    \begin{equation}
        \label{asymptotic linear convergence bound}
        \begin{aligned}
            &\left\| q_*-q_{k} \right\| _{\infty}\\\le& \frac{1}{1-\gamma}\frac{1+\frac{\delta}{1-\gamma}}{1-\frac{\delta}{1-\gamma}}\left( \gamma ^M \right) ^k \\ & \cdot \min \left\{ \left\| F_{\Omega}\left(q_0\right) \right\| _{\infty},(1+\gamma )\left( 1-\frac{\delta}{1-\gamma} \right) \left\| q_0-q_* \right\| _{\infty} \right\}
            \end{aligned}
    \end{equation}
    holds.
    So we conclude that the asymptotic linear convergence rate for RMPI is $\gamma^M$.
\end{proof}

\textbf{Comparison with Prior Art.} To the best of our knowledge, we are the first to establish that RMPI enjoys an asymptotic linear convergence rate of $\gamma^M$. Prior work \cite{geist2019theory} demonstrated a linear convergence rate of $\gamma$, a bound that notably leaves the influence of the step number $M$ unaccounted for. Our result theoretically quantifies the benefit of performing more steps (larger $M$) in regularized policy evaluation.

\begin{remark}
In this work, we focus on dynamic-programming algorithms (RPI and RMPI) for RMDPs with a fixed parameter $\lambda$. We do not consider schemes that anneal $\lambda$ across iterations to approach the unregularized optimum.

Technically, we view the regularized Bellman operator as arising from a smoothing of the $\max$ operator, which enables a Newton-type analysis. This view is strictly an analysis tool; it does not imply the RMDP is merely a smooth approximation of the original MDP. For a fixed $\lambda$, the objective is well-defined and encodes explicit preferences, such as exploration or robustness.
Indeed, in many modern RL pipelines, such as the deep RL library Tianshou \cite{weng2022tianshou}, $\lambda$ is treated as a hyperparameter that is tuned and held fixed rather than annealed. Furthermore, regularization can be viewed as a form of robustification, where the regularized optimum coincides with the solution to an adversarially perturbed control problem in a suitable dual formulation~\cite{husain2021regularized}.
Consequently, convergence results for fixed-$\lambda$ problems are practically and theoretically well-motivated.
\end{remark}

\section{Algorithm Design}

In this section, we leverage the NR perspective of RPI to design an accelerated algorithm with local third-order convergence. This is inspired by higher-order Newton-type methods from numerical analysis \cite{shamanskii1967modification, petkovic2014multipoint}, which attain faster local convergence by reusing Jacobian information from the current iterate in additional correction steps.

We call the proposed method third-order regularized policy iteration (T-RPI). The update rule is formulated as
\begin{equation}
    \label{T-RPI equation}
    \left\{ \begin{array}{l}
        \bar{q}_{k+1}=q_k-F_{\Omega}^{\prime}\left( q_k \right) ^{-1}F_{\Omega}\left( q_k \right)\\
        q_{k+1}=\bar{q}_{k+1}-F_{\Omega}^{\prime}\left( q_k \right) ^{-1}F_{\Omega}\left( \bar{q}_{k+1} \right)\\
    \end{array} \right. ,
\end{equation}
where $F_{\Omega}$ is defined in \eqref{smoothed Bellman equation}, and $F_{\Omega}^{\prime}$ is defined in \eqref{Jacobian of smoothed Bellman equation}. The first step in \eqref{T-RPI equation} is identical to standard RPI, representing a standard Newton update. The second step introduces an additional refinement, explicitly reusing the Jacobian evaluated at $q_k$ to further improve the action value. We show next that T-RPI achieves local third-order convergence.

\begin{theorem}[local cubic convergence of T-RPI]
    \label{local cubic convergence of T-RPI}
    The T-RPI algorithm described in \eqref{T-RPI equation} achieves local cubic convergence in the region
    \begin{equation}
        \label{cubic convergence region}
        \left\| q_*-q \right\| _{\infty} \le \frac{2}{3}\frac{1-\gamma}{\gamma}\lambda\mu.
    \end{equation}
    The iteration error satisfies
    \begin{equation}
        \label{cubic convergence between iterations}
        \left\| q_*-q_{k+1} \right\| _{\infty}\le \left(\frac{3}{2}\frac{\gamma}{1-\gamma}\frac{1}{\lambda\mu}\right)^2\left\| q_*-q_k \right\| _{\infty}^{3}.
    \end{equation}
\end{theorem}

\begin{proof}
    By Theorem \ref{local quadratic convergence of regularized PI}, we have
    \begin{equation}
        \label{quadratic of first step}
        \left\| \bar{q}_{k+1}-q_* \right\| _{\infty}\le \frac{3}{2}\frac{\gamma}{1-\gamma}\frac{1}{\lambda\mu}\left\| q_*-q_k \right\| _{\infty}^{2}.
    \end{equation}
    From $q_{k+1}=\bar{q}_{k+1}-F_{\Omega}^{\prime}\left( q_k \right) ^{-1}F_{\Omega}\left( \bar{q}_{k+1} \right)$, we have
    \begin{equation}
        \nonumber
        q_{k+1}-q_*=\bar{q}_{k+1}-q_*-F_{\Omega}^{\prime}\left( q_k \right) ^{-1}F_{\Omega}\left( \bar{q}_{k+1} \right).
    \end{equation}
    Furthermore,
    \begin{equation}
        \label{error relationship of T-RPI}
        \begin{aligned}
            q_{k+1}-q_*=&\bar{q}_{k+1}-q_*- F_{\Omega}^{\prime}\left( q_k \right) ^{-1}F_{\Omega}^{\prime}\left( q_* \right)  \left( \bar{q}_{k+1}-q_* \right) \\&-F_{\Omega}^{\prime}\left( q_k \right) ^{-1} R_*\left( \bar{q}_{k+1} \right),
        \end{aligned}
    \end{equation}
    where $R_*\left( \bar{q}_{k+1} \right)$ denotes the remainder term
    \begin{equation}
        \nonumber
        R_*\left( \bar{q}_{k+1} \right) =F_{\Omega}\left( \bar{q}_{k+1} \right) -F_{\Omega}\left( q_* \right) -F_{\Omega}^{\prime}\left( q_* \right)  \left( \bar{q}_{k+1}-q_* \right).
    \end{equation}
    From Proposition \ref{properties of the Jacobian's inverse} and \ref{Lipschitz continuity of the Jacobian}, we have $\left\| F_{\Omega}^{\prime}\left( q_k \right) ^{-1} \right\| _{\infty}\le \frac{1}{1-\gamma}$ and $\left\| F_{\Omega}^{\prime}\left( q_k \right) -F_{\Omega}^{\prime}\left( q_* \right) \right\| _{\infty}\le \frac{\gamma}{\lambda\mu}\left\| q_k-q_* \right\| _{\infty}$. From Lemma \ref{lemma for quadratic bound}, we have $\left\| R_*\left( \bar{q}_{k+1} \right) \right\| _{\infty}\le \frac{\gamma}{2\lambda\mu}\left\| \bar{q}_{k+1}-q_* \right\| _{\infty}^{2}$.
    Taking $\left\| \cdot \right\| _{\infty}$ on both sides of (\ref{error relationship of T-RPI}) and plugging in these bounds,
    \begin{equation}
        \label{second-last step}
        \begin{aligned}
            \left\| q_{k+1}-q_* \right\| _{\infty}\le&  \frac{\gamma}{1-\gamma}\frac{1}{\lambda\mu}\left\|q_k -q_*\right\|_{\infty} \left\| \bar{q}_{k+1} -q_* \right\| _{\infty}\\&+ \frac{\gamma}{1-\gamma}\frac{1}{2\lambda\mu}\left\| \bar{q}_{k+1}-q_* \right\| _{\infty}^{2}.
        \end{aligned}
    \end{equation}
    Plugging \eqref{quadratic of first step} and \eqref{cubic convergence region} into \eqref{second-last step} yields \eqref{cubic convergence between iterations}.
\end{proof}

Notably, the local cubic convergence region of T-RPI is also dimension-free, provided the regularizer has a dimension-independent strong convexity constant $\mu$ with respect to $\Vert \cdot \Vert_{1}$.

Compared with standard RPI, T-RPI attains faster local convergence at the cost of extra computation per iteration. In other words, T-RPI typically requires fewer iterations than RPI to reach a given accuracy, but each iteration is nominally more expensive. However, in \eqref{T-RPI equation} the two linear systems share the same coefficient matrix: both are of the form $(I-\gamma P \nabla f_{\Omega}(q_k))x=b$ with $x,b\in \mathbb{R}^{mn}$ (see \eqref{Jacobian of smoothed Bellman equation}). For large state-action spaces, these linear systems are typically solved via matrix factorization. Consequently, one can compute the lower-upper (LU) decomposition of the shared coefficient matrix once and reuse it to solve the second system efficiently. With this reuse, the additional cost of T-RPI becomes minor, and the per-iteration costs of T-RPI and RPI are comparable. As a result, T-RPI can improve wall-clock time by converging in fewer iterations. We validate this advantage empirically in the next section.

\section{Numerical Experiments}

In this section, we conduct numerical experiments to support the theoretical results. First, we verify the local quadratic convergence of RPI, the asymptotic $\gamma^M$ linear convergence of RMPI, and the local cubic convergence of T-RPI. Second, we compare the wall-clock time of T-RPI and RPI to demonstrate the practical advantage of T-RPI.

\subsection{Verification of Convergence Rates}

The regularizer is chosen as Shannon entropy (see Remark \ref{examples of regularization-based smoothed max operators}), so we have $\mu=1$. We set $\lambda=0.5$, $\gamma=0.95$, $|\mathcal{S}|=100$ and $|\mathcal{A}|=20$. For RMPI, we set $M=10$. The Markov transition kernel $P$ is generated at random. The reward $r(s,a)$ is sampled uniformly from $[-1,1]$.
The iteration error is measured as $e_k=\left\| q_*-q_{k} \right\| _{\infty}$. To obtain $q_*$, we run RPI for sufficient iterations until it reaches very high precision. We use Multiprecision Computing Toolbox for MATLAB in our numerical experiments.

For RPI, we denote the coefficient of quadratic convergence in Theorem \ref{local quadratic convergence of regularized PI} as $C = \frac{3}{2}\frac{\gamma}{1-\gamma}\frac{1}{\lambda\mu}$ for simplicity.
As shown in (\ref{total quadratic convergence bound}), $-\ln(\ln(\frac{1}{C e_k}))$ is linear with respect to the iteration number $k$, with the slope being $-\ln(2)$.
We initialize $q_0 \in \mathbb{R}^{|\mathcal{S}||\mathcal{A}|}$ in the local quadratic convergence region presented in Theorem \ref{local quadratic convergence of regularized PI}.
We plot $-\ln(\ln(\frac{1}{C e_k}))$ as a function of the number of iterations $k$.
The results are shown in Fig.~\ref{fig quadratic}. The dashed line has the slope $-\ln(2)$ and passes through the last data point, which represents the theoretical quadratic convergence. We can see that the numerical experiment validates our theoretical result.

\begin{figure}[htbp]
    \centering
    \includegraphics[width=3.0in]{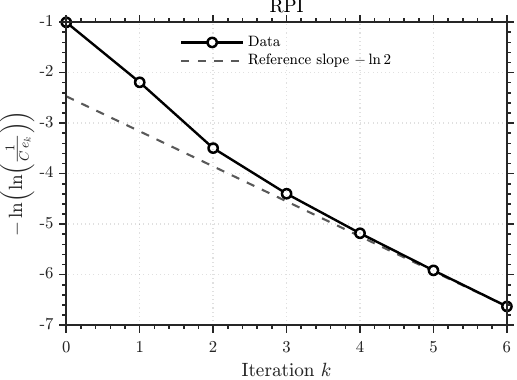}
    \caption{Numerical experiment for RPI: $-\ln(\ln(\frac{1}{C e_k}))$ as a function of the iteration number $k$.}
    \label{fig quadratic}
\end{figure}

For RMPI, as shown in (\ref{asymptotic linear convergence bound}), $\ln(e_k)$ is linear with respect to the iteration number $k$, with the slope being $\ln(\gamma^M)$.
We choose $\Delta=0.02$ and initialize $q_0 \in \mathbb{R}^{|\mathcal{S}||\mathcal{A}|}$ in the local fast linear convergence region presented in Theorem \ref{local fast linear convergence of RMPI}.
We plot $\ln(e_k)$ as a function of the number of iterations $k$. The results are shown in Fig.~\ref{fig linear}. The dashed line has the slope $\ln(\gamma^M)$ and passes through the last data point, which represents the theoretical $\gamma^M$ asymptotic linear convergence. The numerical experiment validates our theoretical result. Note that in this case, we have $\gamma^M=0.95^{10}\approx 0.6$. The convergence speed of RMPI would have been underestimated if we only have the previous $\gamma$ linear convergence result \cite{geist2019theory}.

\begin{figure}[htbp]
    \centering
    \includegraphics[width=3.0in]{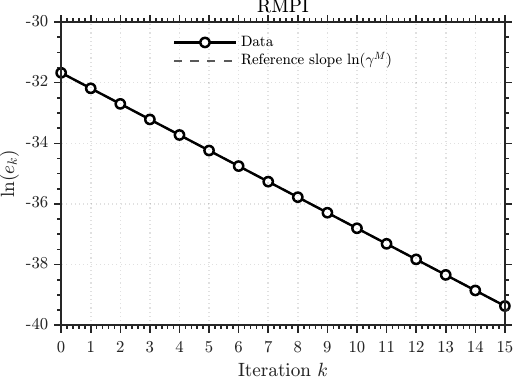}
    \caption{Numerical experiment for RMPI: $\ln(e_k)$ as a function of the iteration number $k$.}
    \label{fig linear}
\end{figure}

For T-RPI, as shown in (\ref{cubic convergence between iterations}), $-\ln(\ln(\frac{1}{C e_k}))$ is linear with respect to the iteration number $k$, with the slope being $-\ln(3)$. $C =\frac{3}{2}\frac{\gamma}{1-\gamma}\frac{1}{\lambda\mu}$.
We initialize $q_0 \in \mathbb{R}^{|\mathcal{S}||\mathcal{A}|}$ in the local cubic convergence region presented in Theorem \ref{local cubic convergence of T-RPI}.
We plot $-\ln(\ln(\frac{1}{C e_k}))$ as a function of the number of iterations $k$.
The results are shown in Fig.~\ref{fig cubic}. The dashed line has the slope $-\ln(3)$ and passes through the last data point, which represents the theoretical cubic convergence. The numerical experiment validates our theoretical result.

\begin{figure}[htbp]
    \centering
    \includegraphics[width=3.0in]{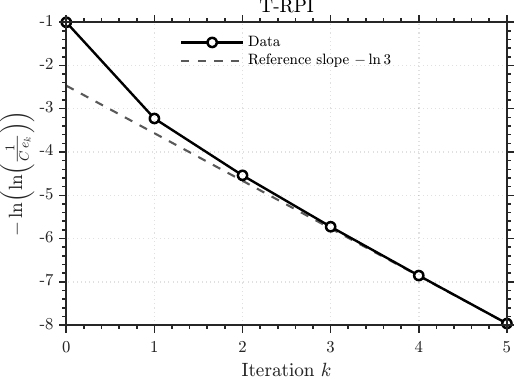}
    \caption{Numerical experiment for T-RPI: $-\ln(\ln(\frac{1}{C e_k}))$ as a function of the iteration number $k$.}
    \label{fig cubic}
\end{figure}

\subsection{Practical Advantage of T-RPI}

We compare the wall-clock time of T-RPI and RPI when they are initialized at the same point and run until the iteration error $e_k$ is smaller than a prescribed threshold. For the linear systems in RPI and T-RPI, since the state-action space is large, we solve them via LU decomposition. For T-RPI, we compute the LU decomposition of the shared coefficient matrix once and reuse it to solve both linear systems in each iteration.
We randomly generate 20 instances of the MDP and report the average wall-clock time for T-RPI and RPI to reach the accuracy threshold. The regularizer, $\lambda$, $\gamma$, $|\mathcal{S}|$, $|\mathcal{A}|$, the reward and MDP generation process are the same as in the previous subsection.
We initialize $q_0 \in \mathbb{R}^{|\mathcal{S}||\mathcal{A}|}$ to the all-zeros vector, and set the threshold for $e_k$ to $10^{-40}$. The results are shown in Fig.~\ref{fig time}. T-RPI achieves an average speedup of $1.3\times$ over RPI, demonstrating the practical advantage of T-RPI in terms of computational efficiency.

\begin{figure}[htbp]
    \centering
    \includegraphics[width=3.0in]{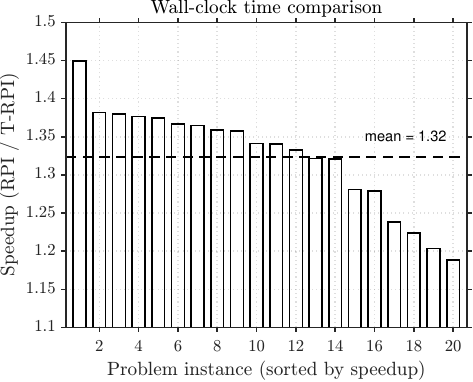}
    \caption{Wall-clock time comparison between T-RPI and RPI. The bar height shows the speedup of T-RPI over RPI, computed as the ratio of the average wall-clock time of RPI to that of T-RPI. Each bar corresponds to a randomly generated MDP.}
    \label{fig time}
\end{figure}

\section{Conclusion}

In this paper, we established that regularized policy iteration is equivalent to the Newton--Raphson method applied to the Bellman equation smoothed by strongly convex regularization. This equivalence not only facilitates a unified convergence analysis of existing methods but also lays the foundation for designing accelerated algorithms. Specifically, we proved that regularized policy iteration converges quadratically in a local neighborhood of the optimal value. We further characterized regularized modified policy iteration as an inexact Newton method with truncated steps, establishing an asymptotic linear convergence rate of $\gamma^M$, where $M$ denotes the number of operator steps carried out in policy evaluation. Furthermore, we proposed a novel algorithm that achieves local third-order convergence and demonstrated its superiority over standard regularized policy iteration in numerical experiments. Collectively, these results advance the theoretical understanding of regularization in reinforcement learning and open up new avenues for algorithmic innovation.

\bibliographystyle{IEEEtran}
\bibliography{reference}

\end{document}